\documentclass[twoside,letterpaper]{article}
\usepackage[preprint]{aistats2027}
\usepackage[T1]{fontenc}
\usepackage[utf8]{inputenc}
\usepackage{amssymb,amsthm,mathtools}
\usepackage{graphicx,booktabs,array,multirow,longtable}
\usepackage{algorithm,algpseudocode}
\usepackage[authoryear,round]{natbib}
\usepackage{xurl}
\usepackage[hidelinks]{hyperref}
\usepackage{microtype}
\newcommand{\calA}{\mathcal A}
\newcommand{\calE}{\mathcal E}
\newcommand{\calT}{\mathcal T}
\newcommand{\Fzero}{\mathcal F_0}
\newcommand{\Prob}{\mathbb P}
\newcommand{\lk}{L}
\newcommand{\uk}{U}
\newcommand{\ut}{\widehat q}
\newcommand{\us}{\widehat s}

\theoremstyle{plain}
\newtheorem{theorem}{Theorem}
\newtheorem{proposition}{Proposition}
\newtheorem{corollary}{Corollary}
\newtheorem{lemma}{Lemma}
\theoremstyle{definition}

\theoremstyle{remark}

\begin{document}
% \runningtitle{Conformal Prediction under Partial Verification}
% \runningauthor{Anonymous Authors}
\twocolumn[
\aistatstitle{Conformal Prediction under Partial Verification}
\aistatsauthor{Zijun Yu\textsuperscript{$*$} \And Yu Gu\textsuperscript{$*$} \And Vahid Partovi Nia \And Masoud Asgharian}
\aistatsaddress{
  McGill University% \\ \texttt{zijun.yu@mail.mcgill.ca}
  \And
  McGill University% \\ \texttt{yu.gu4@mail.mcgill.ca}
  \And
  Polytechnique de Montr\'eal% \\ \texttt{vahid.partovinia@polymtl.ca}
  \And
  McGill University% \\ \texttt{masoud.asgharian2@mcgill.ca}
}
]
\begingroup
\renewcommand{\thefootnote}{\fnsymbol{footnote}}
\footnotetext[1]{Equal contribution.}
\endgroup
\begin{abstract}
Conformal prediction provides prediction sets with finite-sample guarantees, but the label verification required for calibration can be expensive. We develop a partial verification method that returns exactly the same prediction sets as complete verification. We characterize calibration certificates, the verified information sufficient to determine the conformal threshold, and design a procedure that coordinates verification across calibration examples. For finite thresholds at high coverage, its verification cost is less than twice the minimum certificate cost when candidates are checked in order. Across retrieval, mathematical solutions, and configuration evaluation, it reduces verification cost by 15--82\% compared with verifying calibration examples one at a time, while producing identical prediction sets.

\end{abstract}
\section{INTRODUCTION}
\label{sec:intro}
Conformal prediction constructs prediction sets $C(X)$ with a finite-sample marginal coverage guarantee: under exchangeability, $\Prob\{Y_{\rm new}\in C(X_{\rm new})\}\ge1-\alpha$ for a prescribed miscoverage level $\alpha$ \citep{angelopoulos2023conformal}. Its calibration step uses observed labels to evaluate scores and choose a threshold. Obtaining these labels can require costly verification, including human relevance judgments in retrieval \citep{craswell2020trec}, validation of sampled reasoning paths \citep{lightman2024verify, gu2026pause}, and training and evaluation of model configurations \citep{salinas2024tabrepo}.

% Declare the compact float early to favor the top of the first page.
\begin{figure}[!t]
\centering
\includegraphics[width=\columnwidth]{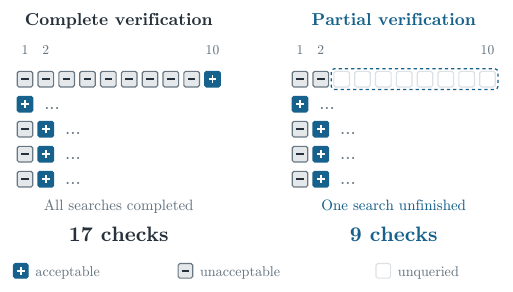}
\caption{\textbf{Fewer verifications, the same prediction set.} Both panels show the same five ranked calibration pools. Here calibration selects the smallest return depth containing an acceptable candidate in at least four pools. Both approaches certify returning the first two candidates on a new input; partial verification leaves the first pool’s search unfinished. Ellipses omit candidates after an observed acceptable one.}
\label{fig:verification-example}
\end{figure}

Consider retrieval: each question has an already ranked pool of passages, and one verification checks whether a passage is relevant. Several passages may be acceptable; the prediction objective is to return a prefix containing at least one of them \citep{fisch2021expanded,quach2024language}. The calibration score for a question is the rank of its first acceptable passage. Complete-score calibration resolves these ranks, takes their conformal quantile, and returns that many top-ranked passages for a new question.

The key observation is that this quantile can be determined while some individual ranks remain unknown. In Figure~\ref{fig:verification-example}, partial verification establishes a return depth of two while leaving one passage search unfinished. Our goal is to recover the same threshold and prediction sets with as few candidate verifications as possible, for a given calibration sample and ranking.

We characterize \emph{calibration certificates}: collections of verified candidate verdicts that determine the threshold for every possible outcome of the remaining verifications. To acquire a certificate, we develop \emph{rank-layer calibration}, which visits unresolved calibration examples at a common candidate rank before moving to the next rank. It stops as soon as the threshold is certified. Writing $C^\star$ for the predictor obtained by complete verification and $\widehat C$ for the predictor returned at completion, we guarantee
\[
 \widehat C(x)=C^\star(x)\qquad\text{for every input }x.
\]

Our theory characterizes the minimum certificate cost for each instance and bounds the work rank-layer spends discovering a certificate. The minimum is a lower bound on every exact method under the same verification rules. When candidates are checked in order at unit cost, rank-layer uses less than twice this minimum for finite thresholds in the high-coverage regime (Theorem~\ref{thm:prefix-optimality}). The sharper factor proved there is optimal among deterministic algorithms in this class. We further relate savings to unfinished searches beyond the threshold and overhead to the scores at its boundary, explaining when the method saves substantial work and stays close to the minimum.

Across five main comparisons on retrieval, mathematical solutions, and configuration evaluation, rank-layer uses 15--82\% fewer calibration queries than serial verification with the same stopping certificate, while returning identical prediction sets. In our MATH sample-size comparisons, partial verification retains the smaller, more stable prediction sets obtained with larger samples at a mean query cost comparable to or below complete verification of a smaller sample.

\section{PROBLEM SETUP}
\label{sec:setup}
\paragraph{Candidates and verification.}
Each input comes with a ranked candidate pool. We seek a prefix containing at least one candidate that meets a fixed acceptance criterion. In retrieval, the input is a question, candidates are passages, and acceptance means relevance to the question.

The calibration data consist of $n$ inputs $X_1,\ldots,X_n$ and their ranked candidate pools; $X_{n+1}$ denotes a new input. For input $X_i$, write $\calA(X_i)=\{a_{i1},\ldots,a_{iK_i}\}$ for its finite pool, with $K_i$ candidates listed in rank order. One \emph{verification}, or query, checks candidate $j$ on input $i$ and reveals its verdict $z_{ij}\in\{0,1\}$, with $z_{ij}=1$ denoting acceptance. Write $O_i=\{a_{ij}:z_{ij}=1\}$ for the acceptable subset. The ranking may be supplied by any method; several candidates may be acceptable, at any positions in the pool.

\paragraph{The reference conformal predictor.}
The calibration score for input $X_i$ is its \emph{first acceptable rank}:
\begin{equation}
 \kappa_i=\min\{j:z_{ij}=1\},\qquad
 \min\varnothing:=+\infty.
 \label{eq:score}
\end{equation}
If $K_i=0$, the score is known to be $+\infty$ without queries. Let $S_r(X)$ contain the first $r$ candidates for input $X$, or the whole pool if fewer are available, with $S_{+\infty}(X)=\calA(X)$. For finite $r$, $\kappa_i\le r$ means exactly that $S_r(X_i)$ contains an acceptable candidate. Complete-score calibration resolves $\kappa_1,\ldots,\kappa_n$ and, at target miscoverage level $1/(n+1)\le\alpha<1$, returns the reference threshold and predictor
\begin{equation}
 \begin{gathered}
 k=\lceil(n+1)(1-\alpha)\rceil,\qquad q=\kappa_{(k)},\\
 C^\star(X)=S_q(X),
 \end{gathered}
 \label{eq:reference}
\end{equation}
where $1\le k\le n$ and $\kappa_{(k)}$ is the $k$th smallest of the $n$ calibration scores. For $0<\alpha<1/(n+1)$, the conformal rule instead sets $q=+\infty$ and returns the whole pool without queries. A finite $q$ is the smallest return depth whose prefixes contain an acceptable candidate in at least $k$ calibration pools; $C^\star$ applies this depth to each new input.

\emph{Exchangeability.} Candidate generation, ranking, and acceptance use the same procedures for all inputs, chosen before calibration. Given these procedures, we assume that the calibration and test examples are exchangeable. The resulting scores $\kappa_1,\ldots,\kappa_{n+1}$ are then exchangeable, giving $\Prob\{\kappa_{n+1}\le q\}\ge1-\alpha$ and, for almost surely nonempty acceptable test pools, the marginal hitting guarantee
\begin{equation}
 \Prob\{S_q(X_{n+1})\cap O_{n+1}\ne\varnothing\}\ge1-\alpha.
 \label{eq:reference-hitting}
\end{equation}
If the acceptable pool is empty, no prefix can hit it, even when both score and threshold are infinite. Without the nonemptiness condition, the valid finite-threshold interpretation is
$\Prob\{q<\infty,\ S_q(X_{n+1})\cap O_{n+1}=\varnothing\}\le\alpha$,
as detailed in Appendix~\ref{app:empty}.

\paragraph{Available information and cost.}
Our task is to \emph{recover $q$ exactly with as few calibration verifications as possible}. The $n$ calibration inputs and their ranked pools are fixed and available, while candidate verdicts are initially unknown. A query $(i,j)$ reveals one verdict $z_{ij}$ at unit cost. A policy $\pi$ chooses queries using the supplied information, revealed verdicts, and independent randomness. Its cost $T_\pi$ is the number of queries; generation and ranking costs are excluded. We analyze \emph{prefix policies}, which check candidates in rank order within each input's pool and choose adaptively which pool to visit next.

\emph{Exact recovery} means that every assignment of the unqueried verdicts yields the same reference threshold. The policy can then return this value of $q$. Section~\ref{sec:calibration} derives a stopping rule from the bounds that observed verdicts place on the first acceptable ranks.

\section{PARTIAL VERIFICATION}
\label{sec:calibration}
Calibration can finish before every score is known. We characterize the evidence needed to certify the calibration threshold and develop a rank-layer procedure to acquire it.

\subsection{Certifying the threshold}
\label{sec:bracket}
Observed verdicts constrain the first acceptable rank in two ways: excluding earlier candidates raises its lower bound, while finding an acceptable candidate supplies an upper bound. After $t$ queries, let $N_i(t)$ and $P_i(t)$ be the sets of ranks verified negative and positive, respectively, on calibration example $i$. Define
\begin{equation}
 \begin{aligned}
 \lk_i(t)&=\min\bigl(\{1,\ldots,K_i\}\setminus N_i(t)\bigr),\\
 \uk_i(t)&=\min P_i(t),
 \end{aligned}
 \label{eq:envelope}
\end{equation}
with $\min\varnothing=+\infty$. Thus $\lk_i(t)$ is the earliest rank not ruled out by a negative verdict, and $\uk_i(t)$ is the earliest verified acceptable rank. These bounds satisfy $\lk_i(t)\le\kappa_i\le\uk_i(t)$; as verification proceeds, the lower bound is nondecreasing and the upper bound is nonincreasing. For $1\le k\le n$, their $k$th order statistics give
\begin{equation}
 q_t^-=\lk_{(k)}(t)\le q\le\uk_{(k)}(t)=q_t^+.
 \label{eq:threshold}
\end{equation}
A \emph{calibration certificate} is a collection of verified verdicts that determines $q$ for every assignment of the unqueried verdicts. The following result characterizes these collections.

\begin{theorem}[Exact stopping certificate]
\label{thm:bracket}\label{prop:bracket}
Fix $1\le k\le n$. At every query count $t$, the minimum and maximum values of $q$ over all verdict tables agreeing with the observations are $q_t^-$ and $q_t^+$, respectively. Hence the observations determine $q$ exactly if and only if
\begin{equation}
 q_t^-=q_t^+.
 \label{eq:bracket}
\end{equation}
\end{theorem}
The certificate separates exact stopping from the choice of queries, providing a common correctness criterion for comparing acquisition policies. The proof appears in Appendix~\ref{app:proofs}.

\subsection{Rank-layer calibration}
Rank-layer calibration coordinates the searches to accumulate lower-bound evidence across the sample as acceptable candidates are found. Algorithm~\ref{alg:layers} processes unresolved examples ($\lk_i<\uk_i$) at a common rank $b$, using any visitation order within the layer, and checks the stopping certificate after each query. We compare it with Serial, which checks the same certificate after each query but verifies calibration examples one at a time, completing each search before starting the next.

\begin{algorithm}[t]
\caption{Rank-layer calibration}
\label{alg:layers}
\begin{algorithmic}[1]
\Require Ranked pools, index $1\le k\le n$, visitation order
\State Initialize $\lk_i,\uk_i$ at $t=0$ using~\eqref{eq:envelope}
\If{$\lk_{(k)}=\uk_{(k)}$}
  \State \Return $\uk_{(k)}$
\EndIf
\For{$b=1,\ldots,\max_i K_i$}
  \For{each $i$ with $\lk_i<\uk_i$, in the chosen order}
    \State Query $z_{ib}$
    \If{$z_{ib}=1$}
      \State $\lk_i\gets b$; $\uk_i\gets b$
    \Else
      \State $\lk_i\gets b+1$ if $b<K_i$, else $+\infty$
    \EndIf
    \If{$\lk_{(k)}=\uk_{(k)}$}
      \State \Return $\uk_{(k)}$
    \EndIf
  \EndFor
\EndFor
\State \Return $\uk_{(k)}$
\end{algorithmic}
\end{algorithm}

Rank-layer returns the same prediction sets as complete verification and can stop with unresolved scores (Figure~\ref{fig:verification-example}). When $q$ is finite, its first finite upper threshold equals $q$, and no candidate above rank $q$ is queried (Appendix~\ref{app:rank-layer}).
After $t$ queries, let $q_t^+$ be the current upper bound on the conformal threshold and define the corresponding prediction set as $C_t(x)=S_{q_t^+}(x)$.
For any acquisition policy, $C_t(x)$ contains the reference prediction set $C^\star(x)$, preserving coverage throughout acquisition.

\begin{proposition}[Coverage throughout acquisition]
\label{prop:budget-validity}\label{thm:budget}
Suppose $1\le k=\lceil(n+1)(1-\alpha)\rceil\le n$ and the exchangeability assumption of Section~\ref{sec:setup} holds. For any acquisition policy, if $\Prob\{O_{n+1}\ne\varnothing\}=1$, then
\begin{equation}
 \Prob\{\forall t\ge0:\ C_t(X_{n+1})\cap O_{n+1}\ne\varnothing\}
 \ge1-\alpha.
 \label{eq:budgetcoverage}
\end{equation}
Without the nonemptiness assumption,
\begin{equation}
 \begin{split}
 \Prob\{\exists t\ge0:\ &q_t^+<\infty,\\
 &C_t(X_{n+1})\cap O_{n+1}=\varnothing\}\le\alpha.
 \end{split}
 \label{eq:finitecommit}
\end{equation}
\end{proposition}
Consequently, the verification history may determine when to stop without an additional coverage adjustment. For rank-layer, a report before exact recovery returns the full candidate pool; other schedules may produce finite conservative thresholds before obtaining a certificate. Appendix~\ref{app:proofs} proves the guarantee, including the empty-pool qualification in Appendix~\ref{app:empty}.

\section{VERIFICATION COMPLEXITY}
\label{sec:complexity}\label{sec:layers}
We characterize the minimum verification needed to recover the calibration threshold and prove that rank-layer attains the optimal worst-case cost ratio among deterministic prefix policies. We then identify the instance properties that govern verification savings and the remaining overhead. Throughout, queries have unit cost and $1\le k\le n$.

\subsection{Minimum verification for exact calibration}
Savings against complete verification alone do not establish how efficiently a policy recovers the threshold. For a fixed verdict table, let $W$ denote the minimum verification work needed for a prefix certificate:
\begin{equation}
 W=\min\{|Q|:Q\text{ is a prefix certificate}\}.
 \label{eq:prefix-certificate-definition}
\end{equation}
Here $Q$ contains verdicts from that table, queried in prefixes, that determine $q$ under every assignment of the unqueried verdicts. Every exact prefix policy must acquire such a certificate and therefore costs at least $W$. This is a hindsight benchmark: the policy must discover sufficient evidence without knowing the hidden verdicts.

For finite $q$, sufficient evidence has two parts. A positive at or before $q$ on $k$ examples establishes that the threshold is at most $q$. Excluding every existing rank below $q$ on $m=n-k+1$ examples establishes that it is at least $q$. An example with first acceptable rank exactly $q$ can supply both. The next theorem reduces the minimum over all prefix certificates to a choice of these two groups.

\begin{theorem}[Minimum prefix certificate]
\label{thm:exact-certificates}
Fix the ranked pools and their verdicts with $q=\kappa_{(k)}<\infty$. Let $d_i=\min\{K_i,q-1\}$ and $m=n-k+1$. Then
\begin{equation}
 W=\min_{A,D}\left\{
 \sum_{i\in A}\kappa_i+
 \sum_{i\in D\setminus A}d_i\right\},
 \label{eq:prefix-certificate}
\end{equation}
where the minimum is over $A,D\subseteq\{1,\ldots,n\}$ satisfying $A\subseteq\{i:\kappa_i\le q\}$, $|A|=k$, $D\subseteq\{i:\kappa_i\ge q\}$, and $|D|=m$.
\end{theorem}
The first term pays to reach the positive witnesses; the second pays for excluded prefixes not already supplied by those searches. Every certificate contains this evidence, and every feasible pair yields a certificate at the displayed cost. The minimum can be evaluated in $O(n\log n)$ from the realized ranks, including ties and unequal pools (Appendix~\ref{app:prefix-formula}).

In Figure~\ref{fig:verification-example}, the four examples with scores $(1,2,2,2)$ provide a minimum certificate costing seven queries. Rank-layer uses nine, compared with seventeen for complete verification: only two of its queries exceed the least sufficient evidence.

Appendix~\ref{app:certificate-costs} characterizes the minimum certificate cost under unrestricted queries and for infinite thresholds.

\subsection{Rank-layer attains the optimal worst-case cost ratio}
Let $T$ be the number of queries rank-layer uses to recover $q$. The following result bounds the cost of discovering a certificate whose minimum size is known only in hindsight.

\begin{theorem}[Optimal deterministic prefix ratio]
\label{thm:prefix-optimality}
Suppose $q<\infty$ and $k\ge m=n-k+1$. Then, for every visitation order within each layer, rank-layer satisfies
\begin{equation}
 W\le T
 \le\left(2-\frac1m\right)W.
 \label{eq:competitive}
\end{equation}
\end{theorem}

\paragraph{Sharpness.}
The factor $2-1/m<2$ is optimal: for each fixed $n,k$ satisfying the theorem's conditions, no deterministic prefix policy that certifies the threshold on every verdict table has a smaller worst-case ratio over finite-$q$ instances. The upper bound also holds for every realization of a randomized visitation order. The matching lower bound concerns deterministic policies, and Serial can still be cheaper on individual instances.

\paragraph{Certificate overlap controls excess verification.}
If the two optimal evidence groups overlap on $t$ examples, their sizes sum to $n+1$, leaving exactly $t-1$ examples outside their union. Rank-layer's excess work comes from exploring these omitted examples and checking rank $q$ on examples used only for lower evidence. Charging this work to the certificate gives the factor in~\eqref{eq:competitive}. Appendix~\ref{app:competitive} gives the proof and a matching adversarial construction.

\subsection{Savings beyond the threshold and overhead at the boundary}
Rank-layer never queries above a finite threshold $q$, whereas complete verification must finish every search. Write $c_i=\min\{\kappa_i,K_i\}$ for the cost of finding example $i$'s first positive or exhausting its pool, and $T_{\mathrm e}=\sum_i c_i$ for the complete-verification cost.

\begin{corollary}[Saved searches and certificate overhead]
\label{cor:cost-structure}\label{thm:cost}
Suppose $q<\infty$. For every visitation order within each layer, rank-layer satisfies
\begin{equation}
 T_{\mathrm e}-T\ge\sum_{i=1}^n(c_i-q)_+.
 \label{eq:actual-saving}
\end{equation}
If exactly one score equals $q$, then
\begin{equation}
 0\le T-W\le n-k.
 \label{eq:unique-gap}
\end{equation}
\end{corollary}
Long searches above $q$ can therefore make the saving large. With a unique boundary score, the excess over minimum evidence is at most one final-rank query per above-threshold example, regardless of how deep $q$ is or how long the unfinished searches are. Appendix~\ref{app:ties} proves these bounds and gives the exact saved-tail identity for a completed final layer.

Boundary ties create alternative sufficient certificates, so some work before the final layer can also be avoidable. To quantify the remaining cost, let $\ell$ and $e$ count the scores below and equal to $q$, respectively. Before layer $q$, rank-layer has spent $B_q=\sum_i\min\{c_i,q-1\}$ queries and needs $h=k-\ell$ more positives. There are $N_q=\#\{i:K_i\ge q,\ \kappa_i\ge q\}$ examples eligible for this layer, of which $e$ are positive there; shorter exhausted pools are excluded.

\begin{proposition}[Expected cost from boundary structure]
\label{prop:final-layer}
Suppose $q<\infty$ and rank-layer uses a uniformly random permutation of the $n$ examples, independent of the calibration sample, as its visitation order. Conditional on the first acceptable ranks $\boldsymbol\kappa$ and pool sizes $\mathbf K$, let $J$ be the position of the $h$th positive among the $N_q$ examples eligible for layer $q$, with $1\le h\le e\le N_q$. Then
\begin{align}
 T&=B_q+J,\label{eq:final-layer-law}\\
 \mathbb E[T\mid\boldsymbol\kappa,\mathbf K]
 &=B_q+\frac{h(N_q+1)}{e+1}.
 \label{eq:final-layer-mean}
\end{align}
\end{proposition}
The positive positions form a uniform $e$-subset of the final layer, giving the displayed expectation. Writing $\Omega=T-W$ for verification overhead, the same representation yields
\begin{equation}
 \Omega
 =\underbrace{B_q+h-W}_{\text{layered work}}
 +\underbrace{J-h}_{\text{final-layer negatives}}.
 \label{eq:overhead-decomposition}
\end{equation}
Both terms are nonnegative: the work before $q$ plus any $h$ boundary positives is itself a certificate. The first term measures excess work imposed by the layer schedule; visitation order affects only the second. With a unique boundary score, the first term vanishes and expected overhead is $(N_q-1)/2$, equal to $(n-k)/2$ when every pool reaches $q$.

Ties affect both certificate size and acquisition cost, so their effect on the cost ratio need not be monotone. Section~\ref{sec:experiments} measures savings and both overhead components, compares observed costs with the conditional expectation, and varies boundary ties while holding saved-tail mass fixed. Appendix~\ref{app:final-layer} gives the full distribution and conditional variance, which quantifies sensitivity to visitation order on a fixed table.

\section{EXPERIMENTS}
\label{sec:experiments}
We evaluate whether partial verification recovers the same predictor with fewer queries, how closely acquisition approaches a minimum calibration certificate, and how calibration sample size changes cost and set size. Our operational setting supplies an existing ranking and charges for marginal calibration verification. All analyses reuse the existing benchmark pools; the certificate, sample-size, and direct-target analyses are retrospective.

\subsection{Tasks and comparisons}
TREC-DL 2019/2020 supplies 97 retrieval queries \citep{nguyen2016msmarco, craswell2020trec,craswell2021trec}. Each candidate pool is the judged intersection of retrieved passages; a query reveals whether a passage has relevance grade at least two. BM25 \citep{robertson2009bm25} and document-identifier orders share the pool. We replay cached relevance judgments on the judged candidate pools and measure verification cost in query counts.

MATH-500 \citep{lightman2024verify} uses published Llama-3.2-1B-Instruct \citep{grattafiori2024llama3} and 3B-Instruct completion pools \citep{h4completions}. We select the first 256 generated completions before ordering by supplied process-reward-model (PRM) scores. A query returns the operational reference-answer verdict of \texttt{math-verify} 0.9.0 \citep{mathverify}. The 1B and 3B pools have empty-target fractions 0.0960 and 0.0640; all problems remain in evaluation. MATH tests whether long searches beyond the calibration quantile can be avoided. Smaller pools, random rankings, and generation seeds are retained controls.

TabRepo contains 244 datasets with 890--1,530 configurations each \citep{salinas2024tabrepo}. A query reveals a configuration's mean validation error $J_x(a)$ over three folds. Acceptance means $J_x(a)-m_x\le .35R_x$, with common supplied optimum $m_x$ and range $R_x$. A consensus order fitted on separate datasets and a random order test how ranking quality affects calibration. Certificate comparisons use the resulting membership bits. A direct-target control uses $J_x(a)\le\varepsilon$ on 117 ROC-AUC tasks, without optimum or range inputs. Recorded evaluation durations provide a separate cost unit.

\begin{table*}[t]
\centering\small
% \caption{Calibration queries for identical prediction sets. The five main rows use 10,000 paired splits; higher-coverage controls use 2,000. Costs include infinite-threshold splits. Savings versus Serial have 95\% paired split-bootstrap intervals. Finite is the fraction with $q<\infty$. Hit is the fraction of splits whose returned set contains an acceptable candidate for the held-out example; test examples without an acceptable candidate count as misses. Size and Hit describe the common reference predictor.}
\caption{Calibration queries and prediction outcomes under exact threshold recovery. Within each row, Exact, Serial, and Layer recover identical thresholds and prediction sets. The upper block reports the main settings; the lower block raises target coverage for TREC and MATH to 90\%. Exact, Serial, and Layer report mean query counts; Size is mean prediction-set size; Finite is the fraction of splits with $q<\infty$. Hit is the fraction whose returned set contains an acceptable candidate for the held-out example; examples without an acceptable candidate count as misses. All means include infinite-threshold splits, which return the full available pool. Savings are computed from mean costs, with 95\% paired split-bootstrap intervals.}
\label{tab:main-cost}
\setlength{\tabcolsep}{3pt}
\begin{tabular}{@{}lrrrrrrrrl@{}}
\toprule
Task / ranking & $(n,\alpha)$ & Splits & Finite & Size & Hit & Exact & Serial & Layer & Saving vs.\ Serial\\
\midrule
TREC / BM25 & $(40,0.2)$ & 10k & 100.00\% & 10.48 & 0.8158 & 265.21 & 260.69 & 173.51 & 33.44\% [33.25,33.63]\\
MATH 1B / PRM & $(300,0.2)$ & 10k & 100.00\% & 31.41 & 0.8072 & 10,960.73 & 10,924.63 & 2,877.86 & 73.66\% [73.58,73.73]\\
MATH 3B / PRM & $(300,0.2)$ & 10k & 100.00\% & 17.05 & 0.8053 & 8,420.40 & 8,385.05 & 1,522.91 & 81.84\% [81.79,81.89]\\
TabRepo / random & $(100,0.1)$ & 10k & 100.00\% & 1.30 & 0.9522 & 124.78 & 118.68 & 100.87 & 15.00\% [14.78,15.25]\\
TabRepo / consensus & $(100,0.1)$ & 10k & 100.00\% & 1.00 & 0.9791 & 243.88 & 225.91 & 92.79 & 58.93\% [58.35,59.49]\\
\midrule
TREC / BM25 & $(40,0.1)$ & 2k & 100.00\% & 16.82 & 0.8985 & 263.70 & 261.99 & 209.00 & 20.23\% [19.85,20.60]\\
MATH 1B / PRM & $(300,0.1)$ & 2k & 57.05\% & 198.49 & 0.9010 & 10,975.33 & 10,656.32 & 9,421.45 & 11.59\% [10.84,12.37]\\
MATH 3B / PRM & $(300,0.1)$ & 2k & 100.00\% & 104.52 & 0.9005 & 8,451.33 & 8,438.97 & 5,085.11 & 39.74\% [39.52,39.96]\\
\bottomrule
\end{tabular}
\end{table*}

We compare three verification policies to distinguish the benefit of stopping early from the benefit of coordinating verification across examples. Exact ($T_e$) resolves every calibration score, checking candidates in rank order until the first acceptable candidate or pool exhaustion. Serial ($T_s$) follows the same example-by-example search but stops once the global certificate determines the threshold. Layer ($T$) instead checks unresolved examples rank by rank, using the same stopping certificate and label-independent example order as Serial.

Within each split, the methods share the calibration sample, ranking, and held-out test example. Table~\ref{tab:main-cost} uses 10,000 paired splits per main setting and 2,000 per higher-coverage control. Intervals use bootstrap resampling of whole splits, keeping the methods paired, and quantify uncertainty from random splitting of the fixed benchmark pools. Appendix~\ref{app:experimental} gives the sample allocations, full protocol, and additional controls.
% Exact ($T_e$) resolves every score through its first positive or exhaustion. Serial ($T_s$) and Layer ($T$) use the same global bracket certificate and label-independent example order; they differ in how they allocate queries across ranks. Fitting/calibration/test-pool sizes are 40/40/17 for TREC, 100/300/100 for MATH, 100/100/44 for relative-target TabRepo, and 54/54/9 for the direct-target control. One held-out example per split supplies the paired prediction outcome. Table~\ref{tab:main-cost} retains 10,000 splits for each main setting and 2,000 for the higher-coverage controls. Intervals use 2,000 paired split-bootstrap resamples and describe Monte Carlo uncertainty conditional on these benchmark pools. The full protocol and all existing controls remain in Appendix~\ref{app:experimental}.

\subsection{Same predictor, less verification}
Layer recovers the same threshold as Exact and Serial on every split, yielding identical prediction sets with 15--82\% fewer calibration queries than Serial across the main settings (Table~\ref{tab:main-cost}). The largest reductions occur on MATH, where Exact and Serial have similar costs but Layer is substantially cheaper. Because Serial and Layer share a stopping certificate, the difference reflects how they allocate verification: Layer establishes the threshold while leaving costly individual searches unfinished.

TabRepo illustrates how the ranking affects this opportunity. Compared with random ordering, the consensus ranking gives Layer smaller prediction sets and lower verification cost, even though Exact and Serial become more expensive. A ranking can therefore make the calibration threshold easier to certify while leaving some individual scores costly to resolve.

Higher target coverage reduces savings in the reported controls, consistent with Corollary~\ref{cor:cost-structure}: deeper thresholds leave less search beyond the threshold for Layer to avoid. On MATH 1B, the threshold frequently becomes infinite. Certifying infinity requires exhausting enough pools with no acceptable candidate, and Serial is cheaper in the all-infinite small-pool control (Appendix~\ref{app:exp-results}). Serial also wins in the synthetic control with no search beyond the threshold (Appendix Figure~\ref{fig:saved-tails}).

\paragraph{Directly verifiable targets.}
TabRepo's absolute-loss control shows the same pattern when measured losses determine acceptance without a supplied optimum or range (Appendix~\ref{app:final-week-absolute}). Strict tolerances usually produce infinite thresholds, which Serial certifies more cheaply. At intermediate tolerances, finite thresholds become common and Layer saves both queries and recorded training time. At the most permissive tolerance, every candidate is acceptable and both methods use the same number of queries. The largest gains arise when enough examples have early positives to certify the threshold while others still require costly searches.

\begin{figure*}[!t]
\centering
\includegraphics[width=.88\textwidth]{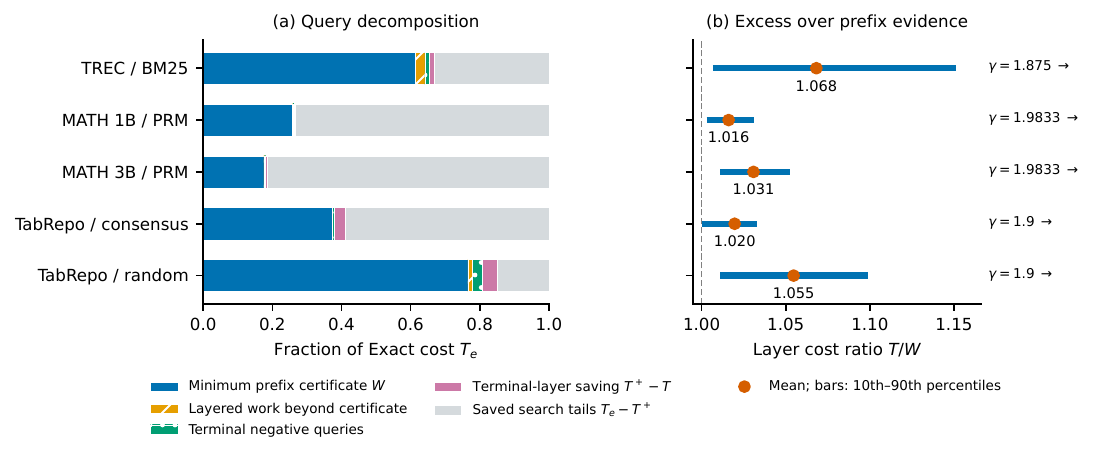}
% \caption{Necessary evidence, acquisition overhead, and avoided searches on the five main settings (10,000 paired splits each). Left: five nonnegative components of Exact cost, normalized by its mean. Layer pays the certificate and the first two overhead components; terminal-layer saving and saved tails are avoided. Right: mean per-split $T/W$ ratios with 10th--90th percentiles. The arrows mark theoretical bounds outside the displayed range.}
\caption{Query savings and overhead relative to minimum prefix evidence. Left: decomposition of complete-verification cost into the minimum prefix certificate, excess layered work, final-layer negative queries, terminal-layer saving, and saved search tails. Right: mean per-split cost ratio \(T/W\), with horizontal bars showing the 10th–90th percentiles. The labels \(\gamma\) give the worst-case bounds from Theorem 3. Results use the 10,000 splits of each main setting in Table 1.}
\label{fig:complexity}
\end{figure*}

\begin{figure*}[!t]
\centering
\includegraphics[width=.88\textwidth]{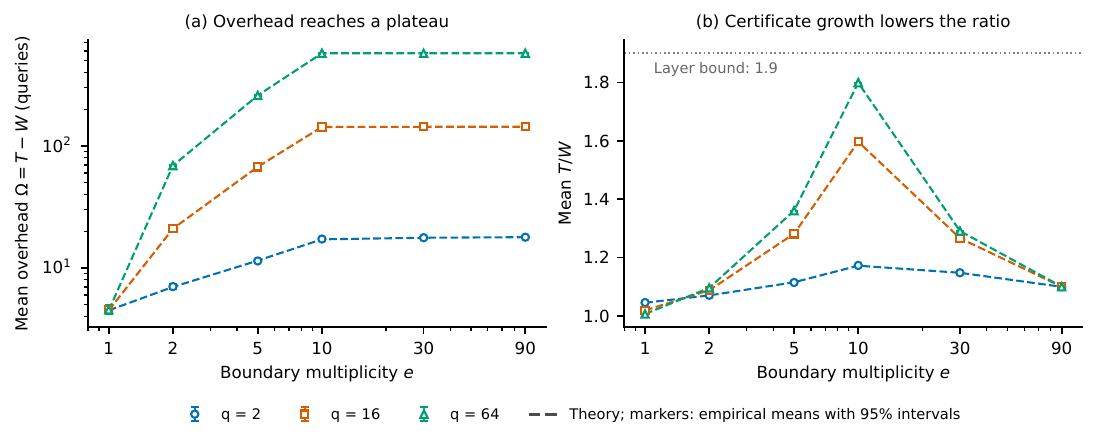}
% \caption{Boundary ties explain overhead at fixed saved-tail mass 72 ($n=99$, $k=90$). Each setting uses the same 200 visitation permutations. The analytical means from Proposition~\ref{prop:final-layer} are overlaid on observed $\Omega=T-W$ and $T/W$; intervals are 95\% paired-permutation bootstrap intervals. Overhead approaches a plateau while certificate cost continues to grow, explaining the declining ratio at large multiplicity.}
\caption{Synthetic test of the predicted verification overhead. We vary threshold depth and the number of boundary ties while holding above-threshold search cost fixed. For each constructed instance, we run Layer over the same 200 visitation permutations and compare mean overhead $\Omega=T-W$ and cost ratio $T/W$ with their analytical predictions from Proposition~\ref{prop:final-layer}. Relative overhead peaks at intermediate tie counts. Intervals are 95\% paired-permutation bootstrap intervals.}
\label{fig:boundary-ties}
\end{figure*}

\begin{figure*}[!t]
\centering
\includegraphics[width=.9\textwidth]{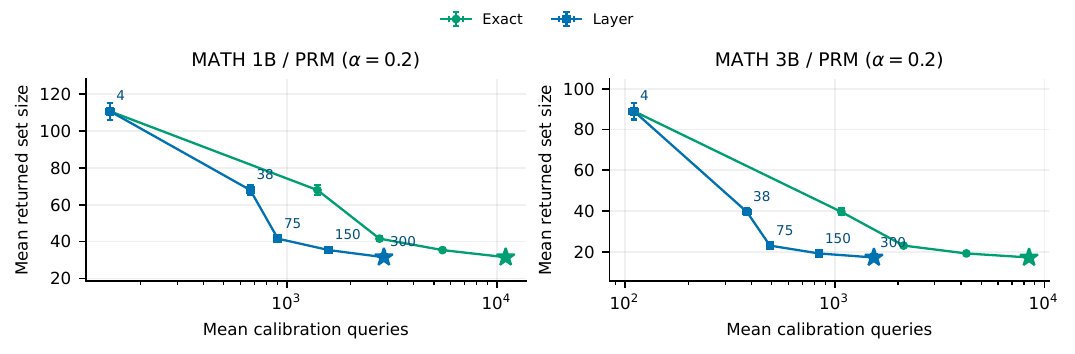}
% \caption{MATH at $\alpha=.2$: Exact and Layer at every predetermined calibration size on 2,000 paired splits. Labels give Layer sample sizes; stars mark full size on both curves. Error bars are pointwise 95\% split-bootstrap intervals. Lines connect discrete expected-cost points, without a hard-budget guarantee.}
\caption{Effect of calibration sample size on verification cost and prediction-set size for MATH at $\alpha=0.2$. Points use 2,000 paired splits. Labels give Layer sample sizes; stars mark full size. Error bars are pointwise 95\% split-bootstrap intervals for the means; lines connect evaluated sizes.}
\label{fig:sample-size}
\end{figure*}

\subsection{Rank-layer stays close to minimum certificate cost}
\label{sec:experiment-overhead}
We reconstruct the minimum prefix certificate cost $W$ from each realized verdict table. Across the five main settings, Layer's mean per-split ratio $T/W$ ranges from 1.02 to 1.07, well below the worst-case bound of Theorem~\ref{thm:prefix-optimality} (Figure~\ref{fig:complexity}b). Thus, on these pools, Layer leaves little room for further verification savings within this binary-verdict prefix-certificate model.

The decomposition in Figure~\ref{fig:complexity} explains both the large savings and the small overhead. On MATH, searches beyond the threshold account for most of Exact's cost, and Layer avoids them entirely, as predicted by Corollary~\ref{cor:cost-structure}. Its remaining overhead consists of excess layered work and final-layer negative queries; both are small relative to the certificate cost. TabRepo consensus illustrates the shallow-threshold case: $q=1$ throughout, so all overhead comes from final-layer negatives. On both MATH pools, the conditional expected costs from Proposition~\ref{prop:final-layer} closely match the observed means, connecting the measured overhead to the boundary structure.

The controlled experiment varies threshold depth $q$ and the number $e$ of scores tied at $q$, keeping above-threshold search cost fixed. With $n=99$ and $k=90$, it places $90-e$ scores at one, $e$ at $q$, and nine at $q+8$. Proposition~\ref{prop:final-layer} gives the expected verification overhead $\Omega=T-W$ over uniformly random visitation orders:
\[
 \mathbb E[\Omega]
 =\underbrace{(\min\{e,10\}-1)(q-1)}_{\text{extra prefix queries}}
 +\underbrace{\frac{9e}{e+1}}_{\text{final-layer negatives}}.
\]
The first term counts extra prefixes that a minimum certificate can omit, each costing $q-1$ queries; the second counts expected negative queries at rank $q$. The number of extra prefixes grows with $e$ but is capped at $n-k=9$. As overhead levels off, further ties increase the certificate cost and reduce $T/W$. This explains why relative overhead first rises and then falls (Figure~\ref{fig:boundary-ties}). Deeper thresholds make the extra prefixes more expensive. The analytical means track these patterns; Appendix~\ref{app:complexity-analysis} gives the full construction and variance comparison.

\subsection{Calibration sample size and verification cost}
\label{sec:experiment-sample-size}
We study what is gained by retaining more calibration examples. Within each split, we run Exact, Serial, and Layer on nested calibration samples of predetermined sizes, keeping the ranking, target coverage, and held-out example fixed. At any given size, all three methods return identical prediction sets. Sets can differ across sizes because the calibration sample and its quantile index change.

On both MATH pools in Figure~\ref{fig:sample-size}, increasing sample size produces smaller mean sets and substantially less variation across splits. This improvement occurs under Exact as well as Layer: it comes from using more calibration examples. Layer's contribution is to reduce verification cost at each of those sample sizes. The savings can be large enough that full-sample Layer costs less than Exact using fewer examples, allowing the larger sample's benefits to be retained at lower cost. Within Layer, however, using fewer examples still costs less. A larger sample therefore requires additional calibration work in exchange for smaller, more stable sets, which can be valuable when returned candidates must be inspected or evaluated at deployment. Appendix~\ref{app:revision-size} reports the full comparisons, including Serial and variation in set size.

\section{Related Work}

\noindent\textbf{Conformal prediction.}
Nested-set calibration \citep{gupta2022nested} provides the
statistical foundation for prediction sets containing an
acceptable output \citep{fisch2021expanded,quach2024language}.
Conformal risk control extends calibration to expected
monotone losses \citep{angelopoulos2024conformal}.

\noindent\textbf{Supervision and label acquisition.}
\citet{cauchois2024predictive} calibrate sets to intersect the
labels compatible with weak supervision.
Active inference selects examples to label for statistical
inference \citep{zrnic2024active}, while \citet{xu2024active}
study active labeling for anytime-valid risk control.
Here candidate verdicts are acquired within a fixed calibration
sample to recover its complete-verification threshold exactly.

\noindent\textbf{Selection under uncertainty.}
\citet{feder2003median} study order-statistic recovery through
paid queries, and \citet{gupta2016update} develop competitive
algorithms for repeated interval refinement.
Here, scores are the outcomes of ranked candidate searches,
and the problem is to coordinate verification
across these searches to recover exactly the conformal
threshold that complete verification would yield.
\section{DISCUSSION}
\label{sec:discussion}
We show how to recover the complete-verification conformal predictor without resolving every score. Our certificate characterization identifies the minimum sufficient evidence, and rank-layer attains the optimal deterministic prefix cost ratio.

Practically, Layer is useful when long score searches can be left unfinished; on our main benchmarks, it avoids these searches with little overhead beyond minimum evidence (Section~\ref{sec:experiment-overhead}). It also lowers the cost of retaining larger calibration samples when they yield smaller, more stable prediction sets (Section~\ref{sec:experiment-sample-size}).

\noindent\textbf{Limitations.}
% The sharp cost guarantee assumes unit-cost prefix queries and finite thresholds. 
The sharp cost guarantee assumes unit-cost prefix queries, finite thresholds, and $k \ge n-k+1$.
Certificates treat queried verdicts as authoritative; robustness to verification errors is not analyzed. Infinite thresholds can favor Serial, and adapting the query schedule to an unknown threshold regime remains an open direction.
% \label{main-end}
\clearpage
\label{post-main-start}
\section*{AI Use Statement}
We used generative AI tools to help develop theoretical models and conceptual frameworks, assist in writing proofs, provide feedback on research methodology and experiments. We also used these tools to draft and edit manuscript text and improve its organization and readability.

We take responsibility for the final content of this work, including all text, claims, and artifacts produced with the assistance of generative AI.

\bibliographystyle{plainnat}
\bibliography{references}
\clearpage
% Working checklist, with questions carried over from the AISTATS 2026 pack.
% Step E: reconcile question wording with the official 2027 pack and check
% package-dependent answers 1(c), 3(a), 4(b), and 4(c) against the intended
% supplement. These answers retain the authors' package descriptions;
% the Step C prose revision does not certify the uploaded artifact inventory.
\section*{Checklist}
\begin{enumerate}
\item For all models and algorithms presented, check if you include:
\begin{enumerate}
\item A clear description of the mathematical setting, assumptions, algorithm, and/or model. [Yes] Sections~\ref{sec:setup}--\ref{sec:complexity}.
\item An analysis of the properties and complexity (time, space, sample size) of any algorithm. [Yes] The verification complexity and certificate computation are analyzed in Section~\ref{sec:complexity} and Appendix~\ref{app:prefix-formula}.
\item (Optional) Anonymized source code, with specification of all dependencies, including external libraries. [Yes] The accompanying source package contains the algorithms, analysis, tests, and dependency specifications.
\end{enumerate}
\item For any theoretical claim, check if you include:
\begin{enumerate}
\item Statements of the full set of assumptions of all theoretical results. [Yes] Section~\ref{sec:setup} and the individual result statements.
\item Complete proofs of all theoretical results. [Yes] Appendix~\ref{app:proofs} and the supplementary frontier in Appendix~\ref{sec:frontier}.
\item Clear explanations of any assumptions. [Yes] Section~\ref{sec:setup} explains exchangeability, the available verification information, and the querying rules; individual results state their additional conditions.
\end{enumerate}
\item For all figures and tables that present empirical results, check if you include:
\begin{enumerate}
\item The code, data, and instructions needed to reproduce the main experimental results (either in the supplemental material or as a URL). [Yes] Source download/preparation scripts, analysis code, and reproduction instructions accompany the paper; public data sources are cited.
\item All the training details (e.g., data splits, hyperparameters, how they were chosen). [Yes] Appendix~\ref{app:exp-splits} specifies splits and fixed settings; only the configuration ranking is fitted locally.
\item A clear definition of the specific measure or statistics and error bars (e.g., with respect to the random seed after running experiments multiple times). [Yes] Section~\ref{sec:experiments}, figure captions, and Appendix~\ref{app:exp-analysis} distinguish bootstrap intervals from descriptive percentiles.
\item A description of the computing infrastructure used. (e.g., type of GPUs, internal cluster, or cloud provider). [Yes] Appendix~\ref{app:complexity-analysis} describes the CPU replay environment.
\end{enumerate}
\item If you are using existing assets (e.g., code, data, models) or curating/releasing new assets, check if you include:
\begin{enumerate}
\item Citations of the creator If your work uses existing assets. [Yes] The benchmark and verifier sources are cited in Section~\ref{sec:experiments}.
\item The license information of the assets, if applicable. [Yes] The source package records the licenses and access conditions stated by the public providers, distinguishing datasets from their associated software.
\item New assets either in the supplemental material or as a URL, if applicable. [Yes] Certificate code, derived result tables, and the synthetic construction are provided.
\item Information about consent from data providers/curators. [Not Applicable] This study reuses public benchmark releases and does not collect new participant data.
\item Discussion of sensible content if applicable, e.g., personally identifiable information or offensive content. [Not Applicable] The analysis operates on benchmark judgments, mathematical acceptance labels, and configuration losses; it does not introduce a new human-content collection.
\end{enumerate}
\item If you used crowdsourcing or conducted research with human subjects, check if you include:
\begin{enumerate}
\item The full text of instructions given to participants and screenshots. [Not Applicable]
\item Descriptions of potential participant risks, with links to Institutional Review Board (IRB) approvals if applicable. [Not Applicable]
\item The estimated hourly wage paid to participants and the total amount spent on participant compensation. [Not Applicable]
\end{enumerate}
\end{enumerate}

\clearpage
\onecolumn
\appendix
\section{PROOFS AND ADDITIONAL COST FORMULAS}
\label{app:proofs}
All query-complexity results allow every assignment of the unobserved verdicts that agrees with the verified observations, as specified in Section~\ref{sec:setup}. This imposes no distributional independence assumption on the verdicts. A certificate determines the quantile value, and ties between calibration scores are allowed. Throughout the proofs, $t=0,1,\ldots$ is the number of queries; after a policy stops, its bounds and prediction sets retain their last values.

\subsection{Validity from sound upper bounds}
\begin{proof}[Proof of Proposition~\ref{prop:budget-validity}]
Let $\Fzero$ record the candidate-generation, ranking, and acceptance procedures chosen before calibration, including any fitted parameters. These procedures may, for example, be learned from a separate training sample. Conditional on $\Fzero$, the examples $(X_i,\calA(X_i),(z_{ij})_{j=1}^{K_i})$, $i=1,\ldots,n+1$, are exchangeable, as required in Section~\ref{sec:setup}. Any randomness in candidate generation is reflected in these pools and verdict vectors. Applying the same first-acceptable-rank map to each verdict vector makes $(\kappa_i)_{i=1}^{n+1}$ exchangeable as well.
For every realization of the data and acquisition randomness, $\uk_i(t)\ge\kappa_i$ for all $i,t$, so
\[
 q_t^+=\uk_{(k)}(t)\ge\kappa_{(k)}=q
 \qquad\text{for every }t.
\]
To apply exchangeability when scores may be tied, introduce independent $V_1,\ldots,V_{n+1}\sim\operatorname{Uniform}(0,1)$, independent of the data and $\Fzero$. Let $R_{n+1}$ be the rank of the test score among all $n+1$ scores, ordering tied scores by their $V_i$ values. Conditional on $\Fzero$, $R_{n+1}$ is uniform on $\{1,\ldots,n+1\}$. Moreover,
\[
 \{R_{n+1}\le k\}
 \subseteq\{\kappa_{n+1}\le q\}
 \subseteq\bigcap_{t\ge0}\{\kappa_{n+1}\le q_t^+\}.
\]
The first inclusion holds because $\kappa_{n+1}>q$ would place at least $k$ calibration scores strictly below the test score. This also covers ties at $+\infty$. Consequently,
\[
 \Prob\{\kappa_{n+1}\le q\mid\Fzero\}
 \ge\Prob\{R_{n+1}\le k\mid\Fzero\}=\frac{k}{n+1}.
\]
Averaging over $\Fzero$ and using $k/(n+1)\ge1-\alpha$ gives
\begin{equation}
 \Prob\{\kappa_{n+1}\le q_t^+\text{ for every }t\}\ge1-\alpha.
 \label{eq:score-budgetcoverage}
\end{equation}
The sharper lower bound is $k/(n+1)$. When the acceptable test pool is nonempty, its first acceptable rank is finite, and $\kappa_{n+1}\le q_t^+$ is equivalent to $C_t(X_{n+1})\cap O_{n+1}\ne\varnothing$. This proves \eqref{eq:budgetcoverage}. The argument requires exchangeability of the true scores, without requiring exchangeability of the adaptively acquired upper bounds. The finite-threshold statement \eqref{eq:finitecommit} follows from the event relation in Appendix~\ref{app:empty}.
\end{proof}

The proof only requires $\uk_i(t)\ge\kappa_i$ for all $i,t$. In the supplementary experiments a verification procedure may supply a growing subset $\widehat O_i(t)\subseteq O_i$ of verified acceptable candidates. Its upper score is
\begin{equation}
 \uk_i(t)=\min\{j:a_{ij}\in\widehat O_i(t)\},
 \qquad\min\varnothing=+\infty.
 \label{eq:upper}
\end{equation}
This satisfies the required upper inequality even if full observation leaves some acceptable candidates uncertified. A shrinking superset of possible acceptable candidates analogously supplies a lower bound. The validity result extends to these bounds; the prefix-certificate minimum in Theorem~\ref{thm:exact-certificates} refers to the exact oracle specified in the main text.

\subsection{Empty targets and finite commitments}
\label{app:empty}
For one test example and one reported threshold $r$, write $H=\mathbf1\{S_r(X)\cap O\ne\varnothing\}$, $C=\mathbf1\{\kappa\le r\}$, $E=\mathbf1\{O=\varnothing\}$ and $V=\mathbf1\{r=+\infty\}$. Then
\begin{equation}
 C-H=EV.
 \label{eq:empty-identity}
\end{equation}
Indeed, when the target is nonempty its first acceptable rank is finite and the two success indicators coincide. When the target is empty, $H=0$, $\kappa=+\infty$, and $C=V$. For finite $r$, a failure is therefore equivalent to $\kappa>r$. Since $q_t^+\ge q$ for every $t$,
\begin{align*}
 &\{\exists t\ge0:q_t^+<\infty,\ C_t(X_{n+1})\cap O_{n+1}=\varnothing\}\\
 &\quad=\{\exists t\ge0:q_t^+<\infty,\ \kappa_{n+1}>q_t^+\}
 \subseteq\{\kappa_{n+1}>q\}.
\end{align*}
The last event has probability at most $\alpha$ by the preceding proof, establishing \eqref{eq:finitecommit}. This bounds the joint probability of a finite report and a failure, without conditioning on whether a finite report is made.

\subsection{Exact identification}
\begin{proof}[Proof of Theorem~\ref{thm:bracket}]
Fix the observed queries and verdicts $Q$. Every assignment of the unobserved verdicts satisfies $\lk_i(Q)\le\kappa_i\le\uk_i(Q)$, hence $\lk_{(k)}(Q)\le q\le\uk_{(k)}(Q)$. Let $\kappa_i^{(1)}$ and $\kappa_i^{(0)}$ denote the scores obtained by setting every unobserved verdict to one and to zero, respectively. The former makes the first rank not verified negative acceptable; the latter leaves only verified positives acceptable. Thus, including exhausted and zero-size pools,
\[
 \kappa_i^{(1)}=\lk_i(Q),\qquad
 \kappa_i^{(0)}=\uk_i(Q)\quad(i=1,\ldots,n).
\]
Both assignments agree with $Q$, and their quantiles are $\lk_{(k)}(Q)$ and $\uk_{(k)}(Q)$. These are therefore the minimum and maximum possible quantiles, proving the stated equivalence.

\end{proof}

\subsection{Correctness of rank-layer calibration}
\label{app:rank-layer}
Suppose first that $q<\infty$ and put $\ell=\#\{i:\kappa_i<q\}<k$. Before layer $q$, at most $\ell$ examples can have a verified positive, so $q_t^+=+\infty$. At the start of layer $q$, every example with $\kappa_i\ge q$ has excluded all its existing ranks below $q$ and therefore has $\lk_i\ge q$. In particular,
\[
 \#\{i:\lk_i\ge q\}\ge n-\ell\ge n-k+1=m,
 \qquad \lk_{(k)}\ge q.
\]
Completing layer $q$ would verify positives on at least $k$ examples. When the count first reaches $k$, the bounds satisfy
\[
 q\le\lk_{(k)}\le\kappa_{(k)}=q\le\uk_{(k)}\le q.
\]
Thus the first finite upper threshold is exact, and the algorithm stops in layer $q$. Theorem~\ref{thm:bracket} rules out an earlier stop at another value.

If $q=+\infty$, querying each unresolved example until its first positive or exhaustion requires at most $\sum_i K_i<\infty$ queries and gives $\lk_i=\uk_i=\kappa_i$ for every $i$. The bracket therefore closes at $+\infty$ within this many queries. Pools of size zero have $\lk_i=\uk_i=+\infty$ at initialization and are never queried. These cases prove termination and exact recovery on every verdict table.

\subsection{Minimum certificate costs}
\label{app:certificate-costs}
\begin{proof}[Proof of Theorem~\ref{thm:exact-certificates}]
For a finite proposed answer $r$, exact identification is equivalently the pair of witness conditions
\begin{equation}
 \#\{i:\uk_i(Q)\le r\}\ge k,
 \qquad
 \#\{i:\lk_i(Q)\ge r\}\ge m.
 \label{eq:certificate-counts}
\end{equation}
The first says $\uk_{(k)}\le r$. Since $m=n-k+1$, the second says $\lk_{(k)}\ge r$. Together with the ordering of the endpoints, these inequalities force both endpoints to equal $r$. Conversely, equality to $r$ implies both counts. Lower evidence on example $i$ requires observing every existing rank below $r$ as negative; an exhausted shorter pool also supplies this evidence. The weak inequalities allow a boundary example to supply both types of evidence.

Apply \eqref{eq:certificate-counts} to the realized finite value $q$. From any prefix certificate $Q$, select $k$ upper-evidence examples $A\subseteq\{i:\kappa_i\le q\}$ and $m$ lower-evidence examples $D\subseteq\{i:\kappa_i\ge q\}$. Reaching a positive on $i\in A$ costs at least $\kappa_i$ queries. On $i\in D\setminus A$, excluding the prefix costs at least $d_i=\min\{K_i,q-1\}$. Since these queries concern disjoint groups of examples,
\[
 |Q|\ge\sum_{i\in A}\kappa_i+\sum_{i\in D\setminus A}d_i.
\]
Conversely, for any feasible $A,D$, query through the first positive on $A$ and through the prefix before $q$ on $D\setminus A$. Every $i\in A\cap D$ has $\kappa_i=q$, so its verified prefix supplies both types of evidence. This attains the objective in \eqref{eq:prefix-certificate}; minimizing proves the formula.
\end{proof}

\paragraph{Arbitrary queries.}
Let $W_{\mathrm f}$ be the minimum number of queries in a certificate without the prefix restriction. Thus $W_{\mathrm f}\le W$, since a prefix certificate is also admissible under unrestricted queries. For finite $q$, put $H=\{i:\kappa_i\ge q\}$ and sort $d_i=\min\{K_i,q-1\}$, $i\in H$, as $d_{[1]}\le\cdots\le d_{[|H|]}$. Then
\begin{equation}
 W_{\mathrm f}=k+\sum_{a=1}^{m}d_{[a]}.
 \label{eq:free-certificate}
\end{equation}
Indeed, any certificate contains a positive at or before $q$ on at least $k$ distinct examples and the full excluded prefix before $q$ on at least $m$ examples in $H$. These negative queries cost at least the sum of the $m$ smallest $d_i$. Positive and negative observations are disjoint queries even if their example indices overlap, giving the lower bound. To attain it, choose the $m$ cheapest examples in $H$ and reveal their prefixes before $q$, then choose any $k$ examples with $\kappa_i\le q$ and reveal one positive at or before $q$ on each. Both groups exist. Any example in their intersection has $\kappa_i=q$, so its positive query is distinct from its negative prefix. The cost equals \eqref{eq:free-certificate} and both counts in \eqref{eq:certificate-counts} hold. When every pool reaches $q$, this simplifies to $W_{\mathrm f}=k+m(q-1)$.

\paragraph{Infinite thresholds.}
When $q=+\infty$, the bracket identifies $q$ if and only if $\lk_{(k)}=+\infty$, equivalently $\#\{i:\lk_i=+\infty\}\ge m$. Each such example requires every candidate in its pool to be verified negative. Thus
\[
 W_{\mathrm f}=W=
 \min_{\substack{D\subseteq\{i:\kappa_i=+\infty\}\\|D|=m}}
 \sum_{i\in D}K_i.
\]
This is the sum of the $m$ smallest pool sizes among examples with no acceptable candidate. Zero-size pools supply this evidence without queries, explaining the initial bracket check in Algorithm~\ref{alg:layers}. The full-pool rule for $\alpha<1/(n+1)$ requires no certificate search.

\subsection{Computing the prefix benchmark}
\label{app:prefix-formula}
The minimum in \eqref{eq:prefix-certificate} has a sorting-based evaluation. For finite $q$, let
\[
 \ell=\#\{i:\kappa_i<q\},\qquad e=\#\{i:\kappa_i=q\},
\]
and sort the below-$q$ scores as $b_1\le\cdots\le b_\ell$. Define $S_j=\sum_{a=1}^j b_a$ and $D_u=\sum_{a=1}^u d_{[a]}$, with $S_0=D_0=0$. Then
\begin{equation}
 W=
 \min_{\max(0,k-e)\le j\le\ell}
 \left\{S_j+(k-j)q+D_{(m-k+j)_+}\right\}.
 \label{eq:prefix-sorted}
\end{equation}
To prove this formula, take $j$ upper-evidence examples below $q$ and $h=k-j$ on the boundary. The cheapest such upper evidence uses the $j$ smallest below-$q$ scores and costs $S_j+hq$. The $h$ chosen boundary examples already supply lower evidence, leaving $r=(m-h)_+$ additional examples. Every boundary example has $d_i=q-1$, the largest possible lower-evidence cost. If $h\ge m$, then $r=0$. Otherwise $|H|-h\ge m-h=r$, so after removing the $h$ selected boundary examples, at least $r$ lower-evidence examples remain. Removing examples with maximal cost leaves the sum of the $r$ smallest remaining costs equal to $D_r$, including ties. Minimizing over the feasible integer values of $j$ proves \eqref{eq:prefix-sorted}. Sorting and cumulative sums take $O(n\log n)$ time.

If every $K_i\ge q$, all $d_i=q-1$, and a minimizing index is
\begin{equation}
 j^\star=\min\{\ell,\max\{\max(0,k-e),k-m\}\}.
 \label{eq:prefix-minimizer}
\end{equation}
Write $f(j)=S_j+(k-j)q+(m-k+j)_+(q-1)$. For adjacent feasible integers $j,j+1$,
\[
 f(j+1)-f(j)=
 \begin{cases}
 b_{j+1}-q<0,&j<k-m,\\
 b_{j+1}-1\ge0,&j\ge k-m.
 \end{cases}
\]
Since $1\le b_{j+1}<q$, the objective is decreasing before $k-m$ and nondecreasing thereafter. Restricting the minimizer to $[\max(0,k-e),\ell]$ gives \eqref{eq:prefix-minimizer}. These calculations use full-information ranks only to evaluate the retrospective benchmark; the acquisition algorithm is not given them.

\subsection{Optimal deterministic prefix ratio}
\label{app:competitive}
For finite $q$, let $T^{+}$ denote the cost of completing rank-layer's terminal layer $q$. Under the conditions of Theorem~\ref{thm:prefix-optimality}, the upper-bound proof establishes the stronger statement
\begin{equation}
 W\le T
 \le T^{+}
 \le\left(2-\frac1m\right)W.
 \label{eq:completed-competitive}
\end{equation}
To state sharpness formally, fix $n,k$ with $k\ge m$ and let $I$ range over finite ranked pools and verdict tables with $q(I)<\infty$. Let $\Pi$ contain the deterministic prefix policies that certify the threshold on every verdict table, including infinite-quantile tables. Then
\begin{equation}
 \inf_{\pi\in\Pi}
 \ \sup_I\frac{T_\pi(I)}{W(I)}
 =2-\frac1m.
 \label{eq:minimax}
\end{equation}

\begin{proof}[Proof of Theorem~\ref{thm:prefix-optimality}]
Appendix~\ref{app:rank-layer} establishes that rank-layer returns $q$ and, for finite $q$, stops during layer $q$. Its immediate-stop cost is therefore no larger than its completed-layer cost. Every terminating exact prefix policy has acquired a prefix certificate, giving $W\le T$.

Take optimal evidence sets $A,D$ in \eqref{eq:prefix-certificate} and put $t=|A\cap D|$. Since $|A|+|D|=n+1$, we have $1\le t\le m$ and exactly $t-1$ examples outside their union. On $A$, completed rank-layer acquisition uses exactly the certificate's prefix cost. On $D\setminus A$, it uses at most one additional query; this extra query is possible only when $K_i\ge q$. Let $a$ count those examples, so $a\le m-t$. Every example outside $A\cup D$ costs at most $q$. Hence
\begin{equation}
 T^{+}-W\le(t-1)q+a,
 \qquad W\ge tq+a(q-1).
 \label{eq:charging}
\end{equation}
The latter inequality counts the $t$ boundary prefixes of cost $q$ and the $a$ lower-only prefixes of cost $q-1$. For $q\ge2$,
\begin{align*}
 (m-1)W-m(T^{+}-W)
 &\ge(m-t)q+a\bigl[(m-1)(q-2)-1\bigr]\\
 &\ge q(m-t)-a\ge(q-1)(m-t)\ge0,
\end{align*}
where the second inequality uses $q\ge2$ and $m\ge1$, and the third uses $a\le m-t$. Rearranging gives $T^{+}\le(2-1/m)W$. This argument covers unequal and zero-size pools. When $m=1$, the intersection has size one, $a=0$, and \eqref{eq:charging} gives $T^{+}=W$ directly. When $q=1$, all lower evidence is free, $W=k$, and a completed layer costs at most $n$. Under $k\ge m$, $n/k=1+(m-1)/k\le2-1/m$, proving the uniform upper bound. In fact, the bound for $q\ge2$ does not require $k\ge m$.

It remains to prove optimality for each fixed $n,k$ with $k\ge m$. The case $m=1$ follows from the certificate lower bound and the upper bound just proved. For $m\ge2$, fix any terminating deterministic prefix policy and an integer $q\ge2$. Give every pool size $q+1$. Choose $k-m$ easy examples whose first rank is positive, and $2m-1$ hard examples whose ranks before $q$ are all negative. Among the hard examples, answer negatively at rank $q$ on the first $m-1$ distinct examples that the policy reaches at that rank, and positively at rank $q$ on the remaining $m$. A negative rank-$q$ example has a positive at $q+1$; all other later labels can be zero.

These responses are consistent with one fixed table for the specified deterministic policy. On that table there are $k-m$ scores equal to one, $m$ equal to $q$, and $m-1$ equal to $q+1$, so the true $k$th score is $q$. Exactly $k$ examples can provide a positive at or before $q$, and a certificate must observe a positive on all of them. The last required hard positive appears only after every hard example has been queried through rank $q$. Therefore
\begin{equation}
 T_\pi\ge(k-m)+(2m-1)q,
 \qquad W=(k-m)+mq.
 \label{eq:adversarial-cost}
\end{equation}
The certificate equality follows because all $k$ examples with scores at most $q$ must supply upper evidence, and resolving them also supplies lower evidence on the $m$ boundary examples. For every integer $q\ge2$, the ratio on the constructed table is at least $[(k-m)+(2m-1)q]/[(k-m)+mq]$. Therefore
\[
 \sup_I\frac{T_\pi(I)}{W(I)}
 \ge\lim_{q\to\infty}\frac{(k-m)+(2m-1)q}{(k-m)+mq}
 =2-\frac1m.
\]
The pool size may vary in the supremum, while $n,k$ stay fixed. Since this holds for every deterministic $\pi$, it proves \eqref{eq:minimax}. The construction depends on the deterministic policy; it does not give a lower bound for randomized policies with data generated independently of their random choices.
\end{proof}

The same charging argument also yields the useful alternative bound
\begin{equation}
 T^{+}\le2W-q
 \qquad(q\ge2),
 \label{eq:two-certificate}
\end{equation}
since $(t-1)q+a\le tq+a(q-1)-q\le W-q$. The finite-$n$ ratio and this subtractive bound need not dominate one another on every instance. For $q=1$ and $k<m$, the appropriate general bound is $T^{+}/W\le n/k$. For example, a rank-one positive found last with $k=1$ can cost $n$ queries although its minimum certificate has size one.

\subsection{Saved tails and boundary ties}
\label{app:ties}
\begin{proof}[Proof of Corollary~\ref{cor:cost-structure}]
By the end of layer $q$, example $i$ has been queried through its first positive, pool exhaustion, or rank $q$, whichever occurs first. Its cost is therefore $\min\{\kappa_i,K_i,q\}=\min\{c_i,q\}$. Summing and subtracting from $T_{\mathrm e}=\sum_i c_i$ gives the completed-layer identities
\begin{align}
 T^{+}&=\sum_i\min\{c_i,q\},\label{eq:layercost}\\
 T_{\mathrm e}-T^{+}
 &=\sum_i(c_i-q)_+.
 \label{eq:saving}
\end{align}
Since $T\le T^{+}$, \eqref{eq:saving} implies \eqref{eq:actual-saving}. The total saving is the displayed tail sum plus the additional saving $T^{+}-T$ from stopping within the final layer.

If exactly one score equals $q$, precisely $k-1$ scores lie below it. All of those examples and the unique boundary example are required for upper evidence. Precisely $m$ examples have scores at least $q$, and all are required for lower evidence. Their necessary prefixes cost $\sum_i\min\{c_i,q-1\}$; one further query obtains the unique rank-$q$ positive. Hence
\begin{equation}
 W=\sum_i\min\{c_i,q-1\}+1.
 \label{eq:unique-certificate}
\end{equation}
Rank-layer has acquired exactly those prefixes by the end of layer $q-1$. Its last layer queries the boundary positive and at most one candidate from each above-$q$ example whose pool reaches $q$. There are at most $n-k$ such examples. Both immediate and completed rank-layer acquisition are therefore within $n-k$ queries of the minimum prefix certificate, proving \eqref{eq:unique-gap}.
\end{proof}

For general boundary multiplicity $e=\#\{i:\kappa_i=q\}$, \eqref{eq:charging} provides the additive refinement
\begin{equation}
 0\le T-W
 \le T^{+}-W
 \le\bigl(\min\{e,k,m\}-1\bigr)(q-1)+(n-k).
 \label{eq:tie-gap}
\end{equation}
Indeed, $t\le\min\{e,k,m\}$, and $a\le m-t$ gives
\[
 (t-1)q+a\le(t-1)(q-1)+m-1.
\]
This includes $q=1$ and the unique-boundary case. It bounds the flexibility introduced by ties; it does not assert a monotone empirical relation between tie count and actual overhead.

The unrestricted benchmark remains different. With equal pool sizes reaching $q$ and all $\kappa_i=q$, high coverage $k\ge m$ gives
\[
 W_{\mathrm f}=k+m(q-1),\qquad
 W=kq,\qquad
 T=n(q-1)+k.
\]
Thus $T/W_{\mathrm f}$ tends to $n/m$ as $q$ grows. A direct positive query can be much cheaper than reaching the same witness by a prefix. The prefix competitive theorem should consequently be interpreted against $W$ even when $W_{\mathrm f}$ is reported alongside it.

\subsection{Final-layer distribution and verification overhead}
\label{app:final-layer}
Under the conditions of Proposition~\ref{prop:final-layer}, the conditional variance is
\begin{equation}
 \operatorname{Var}(T\mid\boldsymbol\kappa,\mathbf K)
 =\frac{h(N_q+1)(e-h+1)(N_q-e)}{(e+1)^2(e+2)}.
 \label{eq:final-layer-variance}
\end{equation}
All earlier-layer counts are fixed by the table, so this variance measures scheduling randomness. Completing the final layer gives $T^{+}=B_q+N_q$.

\begin{proof}[Proof of Proposition~\ref{prop:final-layer}]
Condition on the realized first-positive ranks $\boldsymbol\kappa$ and pool sizes $\mathbf K$, with $q<\infty$. Before layer $q$, exactly $\ell$ positive witnesses can be collected. Each row contributes $\min\{c_i,q-1\}$ queries, independently of visitation order, for total cost $B_q$. Rows reaching layer $q$ are exactly those counted by $N_q$. Their $e$ boundary positives provide the remaining $h=k-\ell$ upper witnesses, where $1\le h\le e$. At the start of this layer every row with $\kappa_i\ge q$ already supplies lower evidence for $q$. The bracket therefore closes precisely on its $h$th positive, proving \eqref{eq:final-layer-law}.

A uniform permutation of all rows induces a uniform permutation on these $N_q$ rows, including when the same permutation is reused at every layer. Thus the positive positions form a uniform $e$-subset of $\{1,\ldots,N_q\}$. For the $h$th such position,
\begin{equation}
 \Prob\{J=j\mid\boldsymbol\kappa,\mathbf K\}
 =\frac{\binom{j-1}{h-1}\binom{N_q-j}{e-h}}{\binom{N_q}{e}},
 \qquad h\le j\le N_q-e+h.
 \label{eq:final-layer-pmf}
\end{equation}
The numerator chooses the $h-1$ positive positions preceding $j$ and the $e-h$ following it. The conditioning is on $(\boldsymbol\kappa,\mathbf K)$; further conditioning on earlier queries could reveal information about the visitation order and change this distribution.

For completeness, let $R=N_q-e$ and $d=e+1\ge2$. All expectations and variances in the remainder of this calculation are conditional on $(\boldsymbol\kappa,\mathbf K)$. The numbers $G_0,\ldots,G_e$ of negatives before, between and after the positives are uniform over the nonnegative integer vectors whose coordinates sum to $R$. There are $C=\binom{R+d-1}{d-1}$ such vectors, each corresponding to one set of positive positions. Their joint probability generating function is
\[
 \mathbb E\!\left[\prod_{a=0}^{e}u_a^{G_a}\right]
 =\frac{[z^R]\prod_{a=0}^{e}(1-u_a z)^{-1}}{C},
\]
where $[z^R]$ denotes the coefficient of $z^R$. Symmetry gives $\mathbb E[G_a]=R/d$. Differentiating twice with respect to one $u_a$, or once each with respect to $u_a,u_b$, and setting all $u_a=1$, gives
\begin{align*}
 \mathbb E[G_a(G_a-1)]
 &=\frac{[z^R]\,2z^2(1-z)^{-(d+2)}}{C}
 =\frac{2\binom{R+d-1}{d+1}}{C},\\
 \mathbb E[G_aG_b]
 &=\frac{[z^R]\,z^2(1-z)^{-(d+2)}}{C}
 =\frac{\binom{R+d-1}{d+1}}{C}\qquad(a\ne b).
\end{align*}
The numerators are zero when $R<2$. Simplifying these ratios yields
\[
 \mathbb E[G_a]=\frac Rd,\qquad
 \mathbb E[G_a(G_a-1)]=\frac{2R(R-1)}{d(d+1)},\qquad
 \mathbb E[G_aG_b]=\frac{R(R-1)}{d(d+1)}\quad(a\ne b).
\]
Consequently,
\[
 \operatorname{Var}(G_a)=\frac{R(d-1)(R+d)}{d^2(d+1)},\qquad
 \operatorname{Cov}(G_a,G_b)=-\frac{R(R+d)}{d^2(d+1)}.
\]
Since $J=h+\sum_{a=0}^{h-1}G_a$, its expectation is $h+hR/d=h(N_q+1)/(e+1)$, and its variance is $h(d-h)R(R+d)/(d^2(d+1))$. Substituting $d=e+1$ and $R=N_q-e$ proves \eqref{eq:final-layer-mean} and \eqref{eq:final-layer-variance}, including the deterministic case $R=0$.
\end{proof}

For any exact prefix policy $\pi$, define verification overhead by
\begin{equation}
 \Omega_\pi=T_\pi-W\ge0.
 \label{eq:verification-overhead}
\end{equation}
To justify \eqref{eq:overhead-decomposition}, querying every pre-boundary prefix and then any $h$ boundary positives supplies $k$ upper witnesses and at least $m$ lower witnesses. Its cost $B_q+h$ is therefore at least $W$. Also $J\ge h$. If $e=1$, then $h=1$ and \eqref{eq:unique-certificate} yields $W=B_q+1$; $J$ is uniform on $\{1,\ldots,N_q\}$. Thus overhead has conditional mean $(N_q-1)/2$ and variance $(N_q^2-1)/12$.

For the constructed tie experiment with $q\ge2$, $\ell=90-e$, $h=e$, $N_q=e+9$ and $B_q=90-e+(e+9)(q-1)$. Formula~\eqref{eq:prefix-sorted} yields
\begin{equation}
 W=\begin{cases}
 80+10q,&1\le e\le10,\\
 90+e(q-1),&10\le e\le90.
 \end{cases}
 \label{eq:synthetic-tie-certificate}
\end{equation}
It follows that
\begin{equation}
 \mathbb E[\Omega]
 =(\min\{e,10\}-1)(q-1)+\frac{9e}{e+1}.
 \label{eq:synthetic-tie-overhead}
\end{equation}
The first component stops increasing at $e=10$; the second approaches nine, while the certificate cost then grows with $e$. This explains a declining cost ratio at large tie multiplicity without requiring a declining absolute overhead.

\subsection{Recovery gaps at early stopping}
\label{app:recovery-gap}
\begin{corollary}[Certified recovery gap]
\label{cor:recovery-gap}
Fix $1\le k\le n$. At any query count $t$ with $q_t^+<\infty$, let $\Delta_t=q_t^+-q_t^-$. Then
\begin{equation}
 0\le q_t^+-q\le\Delta_t.
 \label{eq:recovery-gap}
\end{equation}
For any $\delta\ge0$, at every such $t$ with $\Delta_t\le\delta$,
\[
 0\le |S_{q_t^+}(x)|-|S_q(x)|\le\delta
 \qquad\text{for every }x.
\]
Under the assumptions of Proposition~\ref{prop:budget-validity}, its simultaneous coverage and finite-threshold error bounds remain valid when the verification history determines the reporting time.
\end{corollary}
This trades verification for conservatism relative to the same predictor, without an extra miscoverage allowance. A set-size version for infinite endpoints follows below.

\begin{proof}[Proof of Corollary~\ref{cor:recovery-gap}]
Theorem~\ref{thm:bracket} gives $q_t^-\le q\le q_t^+$. When the upper endpoint is finite, both endpoints and $q$ are finite, so subtraction proves \eqref{eq:recovery-gap}. Proposition~\ref{prop:budget-validity} bounds probabilities of events defined simultaneously over all $t$, so the same bounds apply when the reporting time is selected from the verification history. Nested ranked sets satisfy
\[
 0\le |S_{q_t^+}(x)|-|S_q(x)|
 \le |S_{q_t^+}(x)|-|S_{q_t^-}(x)|
 \le q_t^+-q_t^-.
\]
Thus a nonnegative tolerance $\delta$ bounds the additional number of candidates. Once finite, the gap is nonincreasing because the upper endpoint decreases and the lower endpoint increases.
\end{proof}

With either endpoint infinite, a finite input-specific bound is
\begin{equation}
 G_t(x)=\min\{K(x),q_t^+\}-\min\{K(x),q_t^-\},
 \label{eq:set-size-recovery-gap}
\end{equation}
using $\min\{K,\infty\}=K$. The same nesting argument bounds excess size by $G_t(x)$ without subtracting infinities. This is a set-size certificate, which may be zero after pool saturation even when the numerical threshold is undetermined.

For finite endpoints, the rank-gap criterion is sharp for a threshold that must dominate the reference for every compatible label completion: the completion attaining $q_t^+$ requires any such threshold to be at least $q_t^+$, and the completion attaining $q_t^-$ then incurs excess at least $q_t^+-q_t^-$. This concerns approximation of the specified reference, rather than a lower bound on all valid prediction procedures. Under rank-layer acquisition, the correctness proof in Appendix~\ref{app:rank-layer} shows that the first finite upper endpoint closes the bracket. Other prefix policies can reach unequal depths and produce finite nonexact intermediate thresholds.

\subsection{Stable discrete quantiles and asymptotic verification cost}
\label{app:discrete-asymptotics}
These supporting results assume an infinite sequence of independent and identically distributed examples conditional on fixed fitted components, which is stronger than the exchangeability needed for validity. In this subsection, all distributions and expectations refer to that conditional population law, and the sample of size $n$ consists of the first $n$ examples. Suppose $K\le K_{\max}<\infty$ almost surely, $p=1-\alpha\in(0,1)$ is fixed, and $F(r)=\Prob(\kappa\le r)$ has a finite integer $q_\star$ with
\begin{equation}
 F(q_\star-1)<p<F(q_\star).
 \label{eq:strict-discrete-margin}
\end{equation}
Write $q_n$ for the empirical conformal threshold with $k_n=\lceil(n+1)p\rceil$, using $q_n=+\infty$ and $T=0$ when $k_n=n+1$, and let $c=\min\{\kappa,K\}$. All costs below depend on $n$. For samples with $q_n=+\infty$, define $T^{+}=T_{\mathrm e}$; for finite $q_n$, $T^{+}$ is the cost of completing layer $q_n$ as before.

\begin{proposition}[Stable discrete cost limit]
\label{prop:discrete-asymptotics}
Under these conditions, $q_n\xrightarrow{\mathrm{a.s.}}q_\star$ as $n\to\infty$, and
\begin{align}
 \frac{T_{\mathrm e}}n&\xrightarrow{\mathrm{a.s.}}\mathbb E[c],&
 \frac{T^{+}}n&\xrightarrow{\mathrm{a.s.}}\mathbb E[\min\{c,q_\star\}],
 \label{eq:completed-cost-limits}\\
 \frac{T_{\mathrm e}-T^{+}}n
 &\xrightarrow{\mathrm{a.s.}}\mathbb E[(c-q_\star)_+].
 \label{eq:saved-tail-limit}
\end{align}
For each $n$, let the visitation order be a uniformly random permutation independent of the calibration sample. Put $a=F(q_\star-1)$, $b=F(q_\star)-a>0$, and $s=\Prob(K\ge q_\star,\kappa\ge q_\star)$. Then, with probability taken over the calibration sample and visitation order,
\begin{equation}
 \frac{T}n
 \xrightarrow{\Prob}\mathbb E[\min\{c,q_\star-1\}]
 +s\frac{p-a}{b}.
 \label{eq:actual-layer-cost-limit}
\end{equation}
\end{proposition}
\begin{proof}
Let $F_n(r)=n^{-1}\sum_{i=1}^n\mathbf1\{\kappa_i\le r\}$ and $p_n=k_n/n$. The strong law of large numbers and $p_n\to p$ give
\begin{align*}
 F_n(q_\star-1)-p_n&\xrightarrow{\mathrm{a.s.}}F(q_\star-1)-p<0,\\
 F_n(q_\star)-p_n&\xrightarrow{\mathrm{a.s.}}F(q_\star)-p>0.
\end{align*}
For $k_n\le n$, the definition of the sample quantile gives
\[
 q_n=q_\star
 \quad\Longleftrightarrow\quad
 F_n(q_\star-1)<p_n\le F_n(q_\star).
\]
Since $p<1$, we also have $k_n\le n$ for all sufficiently large $n$. Consequently,
\[
 \Prob\{\exists N<\infty:\ q_n=q_\star\text{ for every }n\ge N\}=1,
\]
which implies $q_n\xrightarrow{\mathrm{a.s.}}q_\star$. On this probability-one event, for every $n\ge N$, \eqref{eq:layercost}--\eqref{eq:saving} yield
\[
 \left(\frac{T_{\mathrm e}}n,\frac{T^{+}}n,
       \frac{T_{\mathrm e}-T^{+}}n\right)
 =\frac1n\sum_{i=1}^n
 \bigl(c_i,\min\{c_i,q_\star\},(c_i-q_\star)_+\bigr).
\]
Each summand is bounded because $c_i\le K_{\max}$, so the strong law proves \eqref{eq:completed-cost-limits}--\eqref{eq:saved-tail-limit}.

To analyze the cost with random visitation order, define counts at the fixed population quantile:
\[
 \begin{aligned}
 B_n&=\sum_{i=1}^n\min\{c_i,q_\star-1\},&
 N_n&=\sum_{i=1}^n\mathbf1\{K_i\ge q_\star,\ \kappa_i\ge q_\star\},\\
 e_n&=\sum_{i=1}^n\mathbf1\{\kappa_i=q_\star\},&
 \ell_n&=\sum_{i=1}^n\mathbf1\{\kappa_i<q_\star\},\qquad h_n=k_n-\ell_n.
 \end{aligned}
\]
Write $\mu_0=\mathbb E[\min\{c,q_\star-1\}]$. Another application of the strong law gives
\[
 \left(\frac{B_n}n,\frac{N_n}n,\frac{e_n}n,\frac{h_n}n\right)
 \xrightarrow{\mathrm{a.s.}}(\mu_0,s,b,p-a).
\]
Let $\mathcal D_n=\sigma((\kappa_i,K_i):1\le i\le n)$ and $\mu_n=\mathbb E[T\mid\mathcal D_n]$. When $q_n=q_\star$, the counts above equal those in Proposition~\ref{prop:final-layer}, so
\[
 \frac{\mu_n}n
 =\frac{B_n}n+
 \frac{(h_n/n)((N_n+1)/n)}{(e_n+1)/n}.
\]
Since $b>0$ and $q_n=q_\star$ for all $n\ge N$ on a probability-one event, this proves $\mu_n/n\xrightarrow{\mathrm{a.s.}}\mu_0+s(p-a)/b$.

Define $A_n=\{q_n=q_\star,\ e_n\ge bn/2\}$; the preceding limits imply $\Prob(A_n^c)\to0$. On $A_n$, $1\le h_n\le e_n\le N_n\le n$, and \eqref{eq:final-layer-variance} gives
\begin{align*}
 \frac{\operatorname{Var}(T\mid\mathcal D_n)}{n^2}
 &=\frac{h_n(N_n+1)(e_n-h_n+1)(N_n-e_n)}
 {n^2(e_n+1)^2(e_n+2)}\\
 &\le\frac{n+1}{n(e_n+2)}
 \le\frac{4}{bn}.
\end{align*}
Here the first inequality uses $h_n(e_n-h_n+1)\le e_n^2\le(e_n+1)^2$ and $(N_n+1)(N_n-e_n)\le(n+1)n$; the second uses $e_n\ge bn/2$ and $n+1\le2n$. Since $A_n$ is determined by $\mathcal D_n$, conditional Chebyshev's inequality implies, for every $\eta>0$,
\begin{align*}
 \Prob\{|T-\mu_n|>\eta n\}
 &\le\Prob(A_n^c)+
 \mathbb E\!\left[\mathbf1_{A_n}
 \frac{\operatorname{Var}(T\mid\mathcal D_n)}{\eta^2n^2}\right]\\
 &\le\Prob(A_n^c)+\frac{4}{b\eta^2n}\longrightarrow0.
\end{align*}
Combining $(T-\mu_n)/n\xrightarrow{\Prob}0$ with the limit of $\mu_n/n$ proves \eqref{eq:actual-layer-cost-limit}.
\end{proof}

The strict margin implies $0<(p-a)/b<1$. Since $\mathbb E[\min\{c,q_\star\}]=\mathbb E[\min\{c,q_\star-1\}]+s$, completing the terminal layer can cost a nonvanishing extra $s[1-(p-a)/b]$ per calibration example. Thus terminal-layer savings must be retained in an asymptotic description of the implemented algorithm. A population atom at $q_\star$ produces linearly many boundary ties, rather than a typical unique boundary. The proposition does not cover $p$ at a CDF jump boundary or growing pool-size bounds.

For context, when $k_n\le n$,
\[
 \{q_n\ne q_\star\}
 =\{F_n(q_\star-1)\ge p_n\}\cup\{F_n(q_\star)<p_n\}.
\]
If $p_n=k_n/n$ lies strictly between $F(q_\star-1)$ and $F(q_\star)$, applying Hoeffding's inequality to each sample proportion and taking a union bound gives
\[
 \Prob(q_n\ne q_\star)
 \le e^{-2n[p_n-F(q_\star-1)]^2}
 +e^{-2n[F(q_\star)-p_n]^2}.
\]
Also, for $p_0=\Prob(\kappa=\infty)$ and $k_n\le n$, the finite-threshold feasibility calculation is
\begin{equation}
 \Prob(q_n=\infty)=\Prob\{\operatorname{Binomial}(n,p_0)\ge n-k_n+1\}.
 \label{eq:empty-pool-feasibility}
\end{equation}
Since $n^{-1}\sum_{i=1}^n\mathbf1\{\kappa_i=\infty\}\xrightarrow{\mathrm{a.s.}}p_0$ and $(n-k_n+1)/n\to\alpha$, the finite-threshold probability tends to one if $p_0<\alpha$ and to zero if $p_0>\alpha$. Fixed benchmark pools sampled without replacement instead require the corresponding hypergeometric law; these IID statements do not turn repeated benchmark splits into independent new prediction tasks.

\subsection{Supplementary frontier conventions}
\label{app:regularity}
The supplementary frontier uses the same finite ranked family with a nested collection of acceptable sets $O_\varepsilon$. The score at each tolerance is their first acceptable rank, and sound certified subsets produce the upper scores in \eqref{eq:upper}. Fitted normalizers and target rules are fixed given $\Fzero$. A finite tolerance grid makes the supremum and simultaneous events measurable automatically. For an infinite index set, the same proof applies when the displayed suprema and simultaneous events are measurable; a separable version with the same suprema on a countable subset is sufficient. An evaluation on a finite grid concerns that grid. With $k=n+1$, the normalized upper threshold is one and every inverse threshold is $+\infty$, so the simultaneous score event is certain.

\subsection{Set-size decomposition}
\label{app:price}
The rank decomposition in Proposition~\ref{prop:price} has a finite-valued set-size version even when an endpoint is infinite. Let $A_x(r)=|S_r(x)|$. For a fixed tolerance and acquisition state, use the supplementary frontier's notation $q_\varepsilon$, $c(\varepsilon)$, $\bar c$, $\us$ and $H_\varepsilon$. Then
\begin{align}
 A_x(H_\varepsilon(\us))-A_x(q_\varepsilon)
 ={}&\bigl[A_x(H_\varepsilon(\bar c))-A_x(H_\varepsilon(c(\varepsilon)))\bigr]\nonumber\\
 &+\bigl[A_x(H_\varepsilon(\us))-A_x(H_\varepsilon(\bar c))\bigr]\nonumber\\
 &+\bigl[A_x(H_\varepsilon(c(\varepsilon)))-A_x(q_\varepsilon)\bigr].
 \label{eq:set-size-price}
\end{align}
The endpoint inequalities proved in Proposition~\ref{prop:price} and monotonicity of $A_x$ give
\[
 A_x(q_\varepsilon)\le A_x(H_\varepsilon(c(\varepsilon)))
 \le A_x(H_\varepsilon(\bar c))\le A_x(H_\varepsilon(\us))\le K(x).
\]
Thus the three terms in \eqref{eq:set-size-price} are finite and nonnegative, and their sum is the displayed left side, including when one or more rank thresholds are infinite.

\section{A simultaneous cost--tolerance frontier}
\label{sec:frontier}
Let $\{O_\varepsilon(X):\varepsilon\in\calE\}$ be a nested family of acceptable sets, with $\calE\ne\varnothing$: increasing $\varepsilon$ enlarges the acceptable set. The verified observations on example $i$ after $t$ queries supply certified acceptable sets at every tolerance. We require these certified sets to be nested in $\varepsilon$ as well, so both $\kappa_i(\varepsilon)$ and $\uk_i(\varepsilon,t)$ are nonincreasing in $\varepsilon$. Here $t=0,1,\ldots$, with the bounds held fixed after querying stops.

Use the extended rank space $\calT=\{0,1,2,\ldots\}\cup\{+\infty\}$. Choose normalizers $G_\varepsilon:\calT\to[0,1]$, fixed given $\Fzero$, that are nondecreasing in their score argument and satisfy $G_\varepsilon(+\infty)=1$. Define
\begin{align}
 s_i^*&=\sup_{\varepsilon\in\calE}G_\varepsilon(\kappa_i(\varepsilon)),\label{eq:truesup}\\
 \us_i(t)&=\sup_{\varepsilon\in\calE}G_\varepsilon(\uk_i(\varepsilon,t)),
 \qquad\us(t)=\us_{(k)}(t).\label{eq:certsup}
\end{align}
Let $k=\lceil(n+1)(1-\alpha)\rceil$ with $0<\alpha<1$; for $k=n+1$ set $\us(t)=1$. Suprema and the simultaneous events are assumed measurable. This is automatic for finite $\calE$; a separable family also suffices (Appendix~\ref{app:regularity}). Use the upper inverse
\begin{equation}
 H_\varepsilon(c)=\sup\{\tau\in\calT:G_\varepsilon(\tau)\leq c\},
 \quad\ut(\varepsilon,t)=H_\varepsilon(\us(t)).
 \label{eq:inverse}
\end{equation}
We include rank $0$ and set $G_\varepsilon(0)=0$. The supremum is taken over $\calT$, so $H_\varepsilon$ returns an integer rank or $+\infty$. In particular, $G_\varepsilon(\tau)\le c$ implies $\tau\le H_\varepsilon(c)$, including when $G_\varepsilon$ has a plateau such as that of an empirical CDF.

\begin{theorem}[Simultaneous frontier]
\label{thm:frontier}
Under the preceding definitions and measurability conditions, suppose the score functions $(\kappa_i(\varepsilon))_{\varepsilon\in\calE}$, $i=1,\ldots,n+1$, are exchangeable conditional on $\Fzero$. Assume that, for every $i\le n$, $\varepsilon\in\calE$, and $t'\ge t\ge0$,
\[
 \kappa_i(\varepsilon)\le\uk_i(\varepsilon,t')\le\uk_i(\varepsilon,t).
\]
Then
\begin{equation}
 \Prob\{\forall t\ge0,\ \forall\varepsilon\in\calE:
 \kappa_{n+1}(\varepsilon)\leq\ut(\varepsilon,t)\}\geq1-\alpha.
 \label{eq:frontiercoverage}
\end{equation}
If $\Prob\{O_\varepsilon(X_{n+1})\ne\varnothing\text{ for every }\varepsilon\in\calE\}=1$, then also
\[
 \Prob\{\forall t\ge0,\ \forall\varepsilon\in\calE:
 S_{\ut(\varepsilon,t)}(X_{n+1})\cap O_\varepsilon(X_{n+1})\ne\varnothing\}\ge1-\alpha.
\]
For each $\varepsilon$, $\ut(\varepsilon,t)$ is nonincreasing in $t$. If $G_{\varepsilon_1}(\tau)\leq G_{\varepsilon_2}(\tau)$ for every $\tau\in\calT$ whenever $\varepsilon_1\leq\varepsilon_2$, then $\ut(\varepsilon_2,t)\le\ut(\varepsilon_1,t)$ for every $t$.
\end{theorem}
\begin{proof}
First suppose $k\le n$, and let $s^*_{(k)}$ be the $k$th order statistic of $s_1^*,\ldots,s_n^*$. Since $\uk_i(\varepsilon,t)\ge\kappa_i(\varepsilon)$ and $G_\varepsilon$ is nondecreasing, $\us_i(t)\ge s_i^*$ and $\us(t)\ge s^*_{(k)}$ for every $t$. Conditional on $\Fzero$, the same measurable map defines each $s_i^*$ from its score function, so $s_1^*,\ldots,s_{n+1}^*$ are exchangeable. The rank argument in the proof of Proposition~\ref{prop:budget-validity} therefore gives
\[
 \Prob\{s^*_{n+1}\le s^*_{(k)}\mid\Fzero\}
 \ge\frac{k}{n+1}\ge1-\alpha.
\]
On this event, simultaneously for all $t\ge0$ and $\varepsilon\in\calE$,
\[
 G_\varepsilon(\kappa_{n+1}(\varepsilon))
 \le s^*_{n+1}\le s^*_{(k)}\le\us(t),
 \qquad
 \kappa_{n+1}(\varepsilon)\le H_\varepsilon(\us(t))=\ut(\varepsilon,t).
\]
Averaging the conditional probability bound over $\Fzero$ proves \eqref{eq:frontiercoverage}. When $k=n+1$, the convention $\us(t)=1$ and $G_\varepsilon(+\infty)=1$ give $\ut(\varepsilon,t)=H_\varepsilon(1)=+\infty$, so the score event has probability one. On the assumed probability-one event that all acceptable test sets are nonempty, every $\kappa_{n+1}(\varepsilon)$ is finite, and its threshold inequality is equivalent to the corresponding prediction set containing an acceptable candidate. This proves the simultaneous coverage assertion.

For $t'\ge t$, the assumed inequalities for the upper scores imply $\us_i(t')\le\us_i(t)$ and $\us(t')\le\us(t)$. Since $H_\varepsilon$ is nondecreasing, $\ut(\varepsilon,t')\le\ut(\varepsilon,t)$. Finally, the ordering of the normalizers gives, for every $c\in[0,1]$ and $\varepsilon_1\le\varepsilon_2$,
\[
 \{\tau\in\calT:G_{\varepsilon_2}(\tau)\le c\}
 \subseteq\{\tau\in\calT:G_{\varepsilon_1}(\tau)\le c\},
 \qquad H_{\varepsilon_2}(c)\le H_{\varepsilon_1}(c).
\]
Taking $c=\us(t)$ proves the stated monotonicity in tolerance.
\end{proof}
One scalar calibration supplies the joint event; its probability bound is independent of the number of tolerances. Computing a continuum requires its actual supremum or a sound upper bound. Our finite-grid experiments evaluate the specified grid, and exact breakpoint enumeration is available for finite piecewise-constant families.

\paragraph{Using the frontier.}
Suppose a deployment search has examined all actions up to entry level $r$. On the event in Theorem~\ref{thm:frontier}, it has reached every nonempty target satisfying $\ut(\varepsilon,t)\leq r$. Consequently users may choose a tolerance after viewing the calibrated frontier. When $\calE$ is finite, the smallest such tolerance is $\min\{\varepsilon:\ut(\varepsilon,t)\leq r\}$ whenever this set is nonempty. This direction of the inequality ensures that the search has reached the required depth.

\subsection{Where simultaneous calibration spends depth}
\label{sec:price}
Fix an acquisition state and suppress $t$. Write
\begin{equation}
 x_i(\varepsilon)=G_\varepsilon(\uk_i(\varepsilon)),\quad
 c(\varepsilon)=x_{(k)}(\varepsilon),\quad
 \bar c=\sup_\varepsilon c(\varepsilon).
 \label{eq:cquantile}
\end{equation}
Here $c(\varepsilon)=G_\varepsilon(\uk_{(k)}(\varepsilon))$ because a nondecreasing map commutes with finite order statistics. Assume $k\leq n$ for the decomposition.

\begin{proposition}[Price of simultaneity]
\label{prop:price}
Fix $1\le k\le n$ and an acquisition state. The normalized thresholds satisfy $\us\geq\bar c\geq c(\varepsilon)$ for every $\varepsilon\in\calE$. Define the dispersion $d_1(\varepsilon)$ and the additional normalized difference $d_2$ by
\begin{equation}
 d_1(\varepsilon)=\bar c-c(\varepsilon),\qquad
 d_2=\us-\bar c.
 \label{eq:normalizedprice}
\end{equation}
Both differences are nonnegative. We have $d_1(\varepsilon)=0$ for every $\varepsilon\in\calE$ if and only if $c(\cdot)$ is constant. Also, $d_2=0$ if and only if
\begin{equation}
 \#\{i\in\{1,\ldots,n\}:x_i(\varepsilon)\le\bar c\text{ for every }\varepsilon\in\calE\}\ge k.
 \label{eq:uniformbottom}
\end{equation}
Let $q_\varepsilon=\uk_{(k)}(\varepsilon)$. For each $\varepsilon$ with $\ut(\varepsilon)<\infty$,
\begin{align}
 \ut(\varepsilon)-q_\varepsilon
 ={}&\underbrace{H_\varepsilon(\bar c)-H_\varepsilon(c(\varepsilon))}_{\text{dispersion after inversion}}\nonumber\\
 &+\underbrace{H_\varepsilon(\us)-H_\varepsilon(\bar c)}_{\text{effect of }d_2}\nonumber\\
 &+\underbrace{H_\varepsilon(c(\varepsilon))-q_\varepsilon}_{\text{difference from inversion}}.
 \label{eq:rawprice}
\end{align}
Every summand is nonnegative.
\end{proposition}
\begin{proof}
For every $i,\varepsilon$, $\us_i\ge x_i(\varepsilon)$. Taking $k$th order statistics and then the supremum over $\varepsilon$ gives $\us\ge c(\varepsilon)$ and $\us\ge\bar c$. Hence both $d_1(\varepsilon)$ and $d_2$ are nonnegative. The dispersion is zero for all tolerances if and only if $c(\varepsilon)=\bar c$ for all $\varepsilon$, which is equivalent to $c(\cdot)$ being constant. Also,
\begin{align*}
 d_2=0
 &\Longleftrightarrow\us\le\bar c
 \Longleftrightarrow\#\{i:\us_i\le\bar c\}\ge k\\
 &\Longleftrightarrow\#\{i:x_i(\varepsilon)\le\bar c\text{ for every }\varepsilon\in\calE\}\ge k,
\end{align*}
proving \eqref{eq:uniformbottom}. Finally, $c(\varepsilon)=G_\varepsilon(q_\varepsilon)$ implies $q_\varepsilon\le H_\varepsilon(c(\varepsilon))$. Monotonicity of $H_\varepsilon$ gives the full ordering
\[
 q_\varepsilon\le H_\varepsilon(c(\varepsilon))
 \le H_\varepsilon(\bar c)\le H_\varepsilon(\us)=\ut(\varepsilon).
\]
When $\ut(\varepsilon)<\infty$, all four values are finite. Their consecutive differences are nonnegative and sum to \eqref{eq:rawprice}.
\end{proof}
The zero conditions concern the normalized gaps in \eqref{eq:normalizedprice}. A plateau of $H_\varepsilon$ can collapse a positive gap after inversion. For finite action pools, the same telescoping decomposition can always be applied to set sizes, including infinite thresholds; Appendix~\ref{app:price} supplies that version.

\begin{lemma}[Shared-normalizer collapse]
\label{lem:collapse}
Suppose $1\le k\le n$, $\varepsilon_0=\min\calE$ exists, and $G_\varepsilon=G$ for every $\varepsilon\in\calE$. Then
\begin{equation}
 \us=\bar c=c(\varepsilon_0),\qquad
 \ut(\varepsilon)=H(c(\varepsilon_0))\quad(\varepsilon\in\calE).
 \label{eq:collapse}
\end{equation}
If, in addition, $\uk_i(\varepsilon_0)=\kappa_i(\varepsilon_0)$ for all $i\le n$ and $H(G(r))=r$ for every possible value $r$ of $\kappa_i(\varepsilon_0)$, then $\ut(\varepsilon)=\kappa_{(k)}(\varepsilon_0)$ for every $\varepsilon\in\calE$.
\end{lemma}
\begin{proof}
For every $i$ and $\varepsilon\in\calE$, $\uk_i(\varepsilon)\le\uk_i(\varepsilon_0)$. Since $G$ is nondecreasing and $\varepsilon_0\in\calE$,
\[
 \us_i=\sup_{\varepsilon\in\calE}G(\uk_i(\varepsilon))
 =G(\uk_i(\varepsilon_0)).
\]
A nondecreasing transformation preserves the order of the observations, including ties. Taking $k$th order statistics therefore gives $\us=G(\uk_{(k)}(\varepsilon_0))=c(\varepsilon_0)$. Moreover, $c(\varepsilon)\le c(\varepsilon_0)$ for every $\varepsilon$, so $\bar c=c(\varepsilon_0)$, proving \eqref{eq:collapse}. Under the additional assumptions,
\[
 \ut(\varepsilon)=H(G(\kappa_{(k)}(\varepsilon_0)))
 =\kappa_{(k)}(\varepsilon_0)\qquad(\varepsilon\in\calE).
\]
\end{proof}
Shared normalization gives $d_2=0$ and one search depth for every tolerance. Tolerance-dependent normalization can produce a varying frontier. Proposition~\ref{prop:price} identifies whether its cost comes from unequal normalized quantiles or from the absence of a common group of $k$ uniformly low curves.

\section{Experimental protocol and supplementary analyses}
\label{app:experimental}
This appendix describes the experiments of Section~\ref{sec:experiments}. The certificate analysis (Appendix~\ref{app:complexity-analysis}), the calibration-size comparison (Appendix~\ref{app:revision-size}), and the absolute-loss comparison (Appendix~\ref{app:final-week-absolute}) were specified after we had seen the results of the other experiments, and they reuse those experiments' splits and verdicts. The supplementary material contains the code, tests, and per-split and summary tables for every setting.

\subsection{Tasks, pools, and verification}
\label{app:exp-data}
\paragraph{TREC-DL.}
We use the passage-ranking relevance judgments and the public top-1,000 candidate passages for the judged queries of TREC-DL 2019 and 2020 \citep{craswell2020trec, craswell2021trec}, deduplicated by query and passage identifier. The candidate pool of a question consists of its top-1,000 candidates that have a relevance grade in $\{0,1,2,3\}$. The two years give 97 questions with pools of 5 to 388 passages. Questions without an acceptable passage stay in the evaluation.

We compute BM25 scores \citep{robertson2009bm25} on the full top-1,000 candidate texts and then restrict the ranking to the judged pool. Text is lowercased and split into maximal ASCII alphanumeric tokens. We use $k_1=1.2$, $b=0.75$, and the nonnegative Robertson IDF $\log[1+(N-d_f+0.5)/(d_f+0.5)]$, summed over distinct query terms, where $N$ is the number of candidates for the question and $d_f$ is the number of them that contain the term. Ties are broken by numeric passage identifier, and the control ranking sorts by that identifier alone. Both rankings are restricted to the judged pool and reindexed as $1,\ldots,K_i$.

A query reveals the grade of one passage. A passage is acceptable when its grade is at least 2, since grade 1 marks related passages that may not answer the question.
%% [FRONTIER] begin
The frontier analysis uses the nested targets $O_\varepsilon=\{a:\operatorname{grade}(a)\geq3-\varepsilon\}$ for $\varepsilon\in\{0,1,2\}$, of which $\varepsilon=1$ is the main target.
%% [FRONTIER] end

\paragraph{MATH-500.}
We use the public HuggingFaceH4 best-of-$N$ completions \citep{h4completions} of Llama-3.2-1B-Instruct and Llama-3.2-3B-Instruct \citep{grattafiori2024llama3} on MATH-500 \citep{hendrycks2021math,lightman2024verify}, generated with temperature 0.8, top-$p$ 1.0, $N=256$, at most 2,048 new tokens, and \texttt{agg\_strategy=last}. The main generation seed is 0, and the 3B seed analysis uses seeds 0--4 on aligned problems. We use the dataset revisions \path{79ec3e22a0ea82b386d72464bce5cba2823b06be} (1B) and \path{ececaacab21564bc1b89c334e14185d8f4e605eb} (3B). Our download log recorded only the branch name \texttt{main}, so we identified these revisions in the repository history by matching the saved dataset cards and the sizes of the six Parquet files. The downloaded Parquet files themselves were not kept for a byte-level comparison.

For $K\in\{32,64,128,256\}$, the pool contains the first $K$ completions in generation order, selected before ranking. The PRM ranking sorts the published \texttt{agg\_scores} \citep{dong2024rlhf} in decreasing order, breaks ties by source index, and places completions without a PRM score last. The random ranking applies a fixed random permutation for each problem. The length ranking, used only in the supplementary 1B analysis, sorts by increasing completion-token count with the same tie rule. Problems are aligned across models and seeds by \texttt{unique\_id}, and their reference answers agree.

Acceptance on MATH is decided by \texttt{math-verify} 0.9.0 \citep{mathverify}, which extracts the final answer from a completion and tests whether it is mathematically equivalent to the reference answer. We use strict comparison without fallback extraction, and a completion with no extractable answer counts as unacceptable. Table~\ref{tab:mathverify} lists the exact settings. All labels were computed with Python 3.11.16 and SymPy 1.14.0.

\begin{table}[t]
\centering
\small
\caption{\texttt{math-verify} 0.9.0 settings for MATH acceptance.}
\label{tab:mathverify}
\begin{tabular}{@{}lp{0.72\linewidth}@{}}
\toprule
Step & Setting\\
\midrule
Reference & Strip whitespace and outer math delimiters, wrap in \texttt{\$\ldots\$}, parse with \texttt{LatexExtractionConfig}\\
Completion & \texttt{LatexExtractionConfig(boxed\_match\_priority=0)}, then \texttt{ExprExtractionConfig}\\
Extraction & \texttt{fallback\_mode="no\_fallback"}, \texttt{extraction\_mode="first\_match"}\\
Comparison & \texttt{float\_rounding=6}, \texttt{numeric\_precision=15}, \texttt{strict=True}, \texttt{allow\_set\_relation\_comp=False}\\
Timeouts & 5\,s for each parse and for \texttt{verify}\\
Errors & \texttt{raise\_on\_error=True} for parsing and \texttt{verify}\\
\bottomrule
\end{tabular}
\end{table}

Nine comparisons exceeded the five-second limit. Six still timed out at 60 seconds and four at 600 seconds, including a 600-second retry with the GMP backend (\texttt{gmpy2} 2.2.1). Our protocol initially left such cases unresolved. After observing these four timeouts, we adopted the library's default handling, which labels them unacceptable. Because their correctness is undetermined, we also evaluate every binary assignment of these four labels (Appendix~\ref{app:exp-results}).

\paragraph{TabRepo.}
We use the configuration table of the TabRepo \texttt{D244\_F3\_C1530} collection \citep{salinas2024tabrepo}. Its \texttt{D244\_F3\_C1530\_30} context loads predictions for only the 30 smallest datasets, but the shared \texttt{configs.parquet} table, from which we read metrics only, covers all 244. A configuration is available on dataset $x$ when \texttt{metric\_error\_val} is finite on folds 0, 1, and 2, and its loss is
\[
 J_x(a)=\frac{1}{3}\sum_{f=0}^2 J^{\rm val}_{x,a,f}.
\]
Every dataset has between 890 and 1,530 available configurations. A query reveals the loss of one configuration and corresponds to three fold evaluations. We also report costs in recorded training seconds.

The consensus ranking gives each available configuration its average tied loss rank on each fitting dataset $x$, normalized to $[0,1]$ as $(r-1)/(K_x-1)$, or zero when $K_x=1$. A configuration unavailable on a fitting dataset receives normalized rank one there. Configurations are sorted by mean normalized rank, with ties broken by configuration identifier, and the order is then restricted to each available pool and reindexed. The random control draws one global permutation of configurations per split, shared across datasets, and restricts it in the same way. Both rankings use fitting data only.
\subsection{Target definitions}
\label{app:exp-targets}
Write $m_x=\min_aJ_x(a)$ and $R_x=\max_aJ_x(a)-m_x$ over the available pool. The relative target is
\[
 O^{\rm nr}_\varepsilon(x)=\{a:J_x(a)-m_x\leq\varepsilon R_x\},
\]
and the main experiments use $\varepsilon=0.35$. When $R_x=0$, every available configuration is acceptable. For every $\varepsilon\geq0$ the target contains a minimizer, so small tolerances give small but nonempty targets. The evaluator supplies $m_x$ and $R_x$ to every policy, so a query determines membership exactly, and the cost of computing these inputs is not counted.
%% [FRONTIER] begin
The frontier analysis uses the grid $\calE_{\rm nr}=\{0.05,0.10,\ldots,0.50\}$ and the relaxed grid $\{0.30,0.35,0.40,0.45,0.50\}$.
%% [FRONTIER] end

The direct-target control is
\[
 O^{\rm abs}_\varepsilon(x)=\{a:J_x(a)\leq\varepsilon\},
 \qquad \calE_{\rm abs}=\{0.05,0.10,0.15,0.20,0.30,0.40,0.50\}.
\]
It requires a loss bounded in $[0,1]$, such as classification error or $1-\mathrm{AUC}$. The metadata contain 117 ROC-AUC, 94 log-loss, and 33 RMSE tasks, so this control uses the 117 ROC-AUC tasks with loss $1-\mathrm{AUC}$. A query then determines membership without the optimum or range.

%% [FRONTIER] begin
\paragraph{Certifiers with slack.}
The certificate-slack study replaces $m_x$ by the lower bound $L_x=m_x-\Delta$ for $\Delta\in\{0,0.02,0.05,0.1,0.2\}$ in raw loss units, and we also report $\Delta/R_x$ when $R_x>0$. A queried configuration is certified acceptable when $J_x(a)\leq L_x+\varepsilon R_x$. It is certified unacceptable when $J_x(a)>U_x(M)+\varepsilon R_x$, where $U_x(M)=\min_{a\in M}J_x(a)$ over the queried set $M$ and $U_x(\varnothing)=\infty$. Unqueried configurations remain possible. For $\Delta>0$, some acceptable configurations can stay uncertified after every configuration is queried. Replacing $L_x$ by the observed minimum at that point would remove this gap, so the study isolates the effect of a fixed lower bound.
%% [FRONTIER] end

\subsection{Splits, randomness, and experiment overview}
\label{app:exp-splits}
Examples are split without replacement into fitting, calibration, and test sets of sizes 40/40/17 for TREC, 100/300/100 for MATH, 100/100/44 for TabRepo, and 54/54/9 for the 117 bounded-metric TabRepo tasks. Each split draws one test example uniformly from its test set and uses it for every method, budget, tolerance, and aligned system.

Splitting, example visitation, ranking, random queries, and resampling use separate random streams, each generated by NumPy \texttt{SeedSequence} from the replication index and a fixed stream identifier. Random-query permutations are indexed by example identifier. Paired comparisons share these streams, and MATH generation seeds index different candidate pools.

\begin{center}
\begin{tabular}{@{}llll@{}}
\toprule
Experiment & Targets and pools & Risk levels & Splits\\
\midrule
Budgeted trajectories & TREC grade $\geq2$, MATH $K=256$, TabRepo NR $0.35$ & $0.1,0.2$ & 2,000\\
Matched thresholds & TREC/BM25, MATH/PRM, TabRepo NR with both rankings & Section~\ref{sec:experiments} & 10,000\\
Matched-threshold controls & same targets, other rankings or sample sizes & $0.1,0.2$ & 2,000\\
Robustness & MATH pool sizes and generation seeds & $0.1,0.2$ & 2,000\\
Calibration sizes & 58 size, risk, and setting cells & $0.1,0.2$ & 2,000\\
Absolute loss & 117 ROC-AUC tasks, seven tolerances & $0.1$ & 2,000\\
%% [FRONTIER] begin
Simultaneous frontier & TREC, TabRepo NR full and relaxed grids, absolute loss & $0.1,0.2$ & 10,000\\
Frontier mechanisms & matched families, normalizers, certificate slack & $0.1,0.2$ & 2,000\\
%% [FRONTIER] end
\bottomrule
\end{tabular}
\end{center}
Section~\ref{sec:experiments} reports the matched-threshold comparison at $\alpha=0.2$ for TREC and MATH and at $\alpha=0.1$ for TabRepo with $\varepsilon=0.35$, together with $\alpha=0.1$ controls for TREC and MATH. Reported runs use replication indices 1,000--2,999 or 1,000--10,999, so each 10,000-split run contains the corresponding 2,000-split run. Pilot runs used indices 0--199, which are disjoint from those of all reported runs.
%% [FRONTIER] begin
The frontier experiment compares all four calibration constructions on TREC/BM25, on normalized regret under both TabRepo rankings, and on absolute loss under the consensus ranking. Absolute loss under the random ranking uses 2,000 splits.
%% [FRONTIER] end

\subsection{Policies, observations, and stopping}
\label{app:exp-policies}
A policy observes example and candidate identifiers, candidate ranks, the verdicts or losses it has queried, and any target side information from Appendix~\ref{app:exp-targets}. Unqueried verdicts and full-information scores are available only to the evaluator. After each query, the observed verdicts update the bounds $\lk_i\leq\kappa_i\leq\uk_i$, and an exhausted pool without an acceptable candidate has $\lk_i=\uk_i=\infty$. A policy stops when $\lk_{(k)}=\uk_{(k)}$, including equality at infinity, which identifies the full-information quantile. The bracket is also checked before the first query, so pools known to be empty can certify $q=\infty$ at no cost, and $k=n+1$ returns $q=\infty$ without queries.

\paragraph{Policies.}
Exact visits examples in a uniformly shuffled, label-independent order and queries each pool in rank order until its first acceptable candidate or exhaustion, resolving every score. Serial uses the same order and the same per-pool search, but checks the certificate after every query and stops once it closes. Layer runs rank-layer calibration, visiting unresolved examples at rank $b$ in the same order and skipping examples with an acceptable candidate or an exhausted pool. It stops at the $k$th acceptable witness, and the shared certificate check can also certify $q=\infty$ earlier. The finite-$q$ cost analysis applies unchanged.

Random-query cycles over unresolved examples and makes one query per visit. Each example has a fixed random permutation of its candidate ranks, and the policy queries the next unqueried candidate in that permutation whose rank is smaller than the current witness rank of the example, or any unqueried candidate if no witness is known. An example is resolved once every rank smaller than its witness rank is excluded or its pool is exhausted. The certificate is checked after each query. Full-vector queries every candidate and gives the complete verdict table.

\paragraph{Budget caps.}
For the trajectory experiments, we record the state of each policy at the global query caps
\[
 \{0,C_{\rm vector}\}\ \cup\
 \{\min(n2^j,C_{\rm vector}):j=0,\ldots,\lceil\log_2K_0\rceil\},
\]
where $C_{\rm vector}=\sum_iK_i$ and $K_0$ is 1000 for TREC, $K$ for MATH, and 1530 for TabRepo. Caps that clip to the same value give the same state, and averages over caps count each distinct cap once. A policy that has stopped keeps its final output at later caps. Since a finite threshold can still return the whole test pool, we record separately the first cap with a finite threshold and the first cap with a set smaller than the full pool. Set-size averages include full-pool outputs.

\subsection{Matched thresholds and expected-cost comparisons}
\label{app:exp-cost}
The matched-threshold comparison runs all policies on the same calibration and test examples. For finite $q$, we check that $T_{\mathrm e}=\sum_i\min(\kappa_i,K_i)$ and that the completed-layer identity of Appendix~\ref{app:ties} holds. A separate replay completes the terminal layer to obtain $T^+$, while $T$ counts queries up to stopping. The saved tail $\sum_i(c_i-q)_+$ and the terminal-layer saving $T^+-T$ sum to the saving against Exact.

We report all-split mean costs, the fraction of splits with finite $q$, and mean costs conditional on finite $q$. The primary saving is the ratio of mean costs $1-\overline T/\overline T_{\mathrm s}$, where $T_{\mathrm s}$ is the cost of Serial, and we also report savings relative to Exact. Equal thresholds on every split imply equal prediction sets under the common ranking.

For calibration-size comparisons we use the distinct sizes
\[
 \{n_{\min},\lceil n/8\rceil,\lceil n/4\rceil,\lceil n/2\rceil,n\}
 \cap\{n_{\min},\ldots,n\},
\]
where $n_{\min}=\lceil1/\alpha\rceil-1$ is the smallest size with $k\leq n$. A sample of size $n'$ consists of the first $n'$ examples of the common label-independent visitation order and uses the index $k'=\lceil(n'+1)(1-\alpha)\rceil$. The comparison covers 58 cells: both TREC rankings, MATH 1B and 3B with PRM ranking at $K=256$, and both relative-target TabRepo rankings, each at $\alpha\in\{0.1,0.2\}$ and at every size (Appendix~\ref{app:revision-size}).

For deployment on a new input $X$, a prefix search costs
\[
 Q_{\rm dep}(X)=\min\{\kappa(X),\,|S_{\ut}(X)|\}
\]
queries, with $\min(\infty,s)=s$ and $S_\infty$ the full pool. This is the search up to the first acceptable candidate or the end of the returned set. In each cost unit, the total cost is $C_{\rm setup}+C_{\rm cal}+N\,\mathbb E[Q_{\rm dep}]$ for $N\in\{1,10,100,1000\}$ deployments, where $C_{\rm setup}$ includes the fitting labels when a ranking is fitted (Appendix~\ref{app:revision-cost}).

%% [FRONTIER] begin (entire subsection)
\subsection{Frontier construction and use}
\label{app:exp-normalizers}
\paragraph{Normalizers.}
The primary normalizer $G_\varepsilon$ is the empirical CDF of full-information fitting scores for each target and ranking, with infinite scores included in its denominator. It is fitted on the score space $\{0,1,\ldots,K_0,\infty\}$ with $G_\varepsilon(0)=0$ and $G_\varepsilon(\infty)=1$, and thresholds use the upper inverse \eqref{eq:inverse}. We compare it with three alternatives. The shared normalizer $G(r)=r/(K_0+1)$ for finite $r$, with $G(\infty)=1$, reproduces the pointwise threshold at the strictest tolerance. Pointwise calibration takes the $k$th smallest upper score separately at each tolerance. Bonferroni uses $k_{\rm Bonf}=\lceil(n+1)(1-\alpha/|\calE|)\rceil$ at every tolerance, with an infinite threshold when $k_{\rm Bonf}>n$. All constructions read the same queried verdicts, so Bonferroni needs no additional queries per tolerance.

The finite-top normalizer is
\begin{equation}
 G^+_\varepsilon(r)=\frac{n_{\rm fit}F_\varepsilon(r)+r/(K_0+1)}{n_{\rm fit}+1},
 \quad r=0,\ldots,K_0,\qquad G^+_\varepsilon(\infty)=1.
 \label{eq:exp-top-normalizer}
\end{equation}
It is strictly increasing on finite ranks, coherent across nested tolerances, and below one at every finite rank, so its discrete round trip is exact and $\infty$ remains a separate top value. For each normalizer we report both normalized price terms and the frequency of finite-score saturation.

\paragraph{Query histories.}
In the main frontier experiment, a shared history queries candidates in rank layers and stops each example at its first acceptable candidate for the strictest target or at exhaustion. It continues past the $k$th such witness until every hitting score in the target family is exact, and all four constructions read this history. The full and relaxed grids each have their own history. The target-family controls and the certificate-slack study instead use a fixed rank-layer trace through every candidate, so all targets and certifiers see the same losses.

We compare each partially observed frontier with the full-information supremum frontier under the same $G$, and that frontier with its full-information pointwise threshold. At each partial state, the three-term decomposition is computed against the pointwise upper-score threshold of that state. Finite endpoints are compared by rank, and all endpoints, including infinite ones, by set size $A_x(r)=|S_r(x)|$. The normalized gaps $d_1(\varepsilon)$ and $d_2$ are recorded separately from their increments after inversion. The crossing frequency at a budget is the fraction of splits with $d_2>10^{-12}$, and summaries over budgets weight each budget equally within a split. The common-bottom-$k$ condition is checked on the normalized scale, because an inverse plateau can map a positive gap to the same rank.

\paragraph{Matched target families.}
On the bounded-metric TabRepo subset and a common fitting split, we choose five tolerances for each target family. For each fraction $f\in\{0.05,0.10,0.20,0.35,0.50\}$, the tolerance is the smallest nonnegative breakpoint satisfying
\[
 \frac{1}{n_{\rm fit}}\sum_{x\in\mathcal I_{\rm fit}}
 \frac{|O_\varepsilon(x)|}{K_x}\geq f.
\]
Breakpoints are fitting-data losses for the absolute family and normalized regrets for the relative family, including zero. The criterion weights datasets equally regardless of $K_x$, and a tolerance selected twice is kept twice. Each family uses its own ECDF from the same fitting examples, and both families share the ranking and the all-candidate query trace. The 10-level and 7-level grids of Appendix~\ref{app:exp-targets} are reported separately.

\paragraph{Choosing a tolerance at a deployment depth.}
Deployment depths are $r\in\{5,10,20\}$ for TREC and $r\in\{10,30,100\}$ for TabRepo. The selected tolerance is
\[
 \varepsilon_r=\min\{\varepsilon\in\calE:\ut(\varepsilon)\leq r\},
\]
and no tolerance is selected when this set is empty. We report how often a tolerance is selected, which one is selected, and the failure indicator $\mathbf1\{\text{a tolerance is selected and }S_r\text{ contains no acceptable candidate}\}$. For frontiers with simultaneous validity, this failure has probability at most $\alpha$ under the conditions of the simultaneous frontier theorem, unconditionally over whether a tolerance is selected. Pointwise curves serve as a per-tolerance reference, and we also measure their joint and selected-tolerance failure rates.
%% [FRONTIER] end

\subsection{Mechanism studies and statistical analysis}
\label{app:exp-analysis}
\paragraph{Controlled saved tails.}
For $n=99$, $\alpha=0.1$, $k=90$, and $K=256$, we set
\[
 \kappa_i=\lceil iq/90\rceil\ (i=1,\ldots,90),\qquad
 \kappa_i=\tau\ (i=91,\ldots,99),
\]
with $q\in\{4,16\}$ and tail rank $\tau\in\{q,4q,16q\}$. Each target is a single candidate at its specified rank. The full-information quantile is $q$ in every setting, and the completed-layer saving is $9(\tau-q)$. Figure~\ref{fig:saved-tails} reports costs over 200 label-independent example orders per setting. When $\tau=q$, the completed-layer saving is zero and any saving comes from stopping within the terminal layer.

\paragraph{Coverage indicators.}
For every output we record
\[
 H=\mathbf1\{S_{\ut}(X)\cap O\ne\varnothing\},\quad
 C=\mathbf1\{\kappa\leq\ut\},\quad
 V=\mathbf1\{\ut=\infty\},\quad
 F=\mathbf1\{|S_{\ut}(X)|=K(X)\},\quad
 E=\mathbf1\{O=\varnothing\}.
\]
These satisfy $C-H=EV$ on every record, so the literal hit rate counts test examples with empty targets as misses even when the full pool is returned. The finite-commitment failure indicator is $\mathbf1\{\ut<\infty,H=0\}$. Since sets shrink as acquisition proceeds, coverage over all budgets for a fixed target equals coverage at the final recorded state. All unconditional summaries include test examples with empty targets.
%% [FRONTIER] begin
For each frontier with simultaneous validity, we record the maximum failure indicator over its budget and tolerance trajectory, separately for the literal and extended indicators.
%% [FRONTIER] end

\paragraph{Statistical analysis.}
Binary rates use two-sided 95\% Wilson intervals. Mean differences, set-size summaries, and ratios of mean costs use 2,000 paired split-bootstrap resamples with 2.5th and 97.5th percentile endpoints. Each resample draws whole splits and keeps all methods, budgets, and tolerances of a split together. For constructed tables, resampling is over visitation permutations instead. All intervals therefore reflect variation over splits or visitation orders with the benchmark pools and constructed tables held fixed, and plotted intervals are pointwise. Medians use the midpoint convention, and 10th and 90th percentiles use the empirical inverse CDF. Results for different MATH generation seeds are reported separately on aligned problems. For reference, the coverage guarantee is $k/(n+1)$ when $k\leq n$.

%% [FRONTIER] begin
\paragraph{Empty-target diagnostics.}
The number of calibration examples with empty targets follows a hypergeometric law when examples are drawn without replacement from a fixed pool. We predict empty-target rates from the fitting data using the conditional law of the remaining pool, averaged over fitting splits. An empty-target trigger makes the supremum frontier vacuous, and under the finite-top and completeness conditions it is the only cause. A supplementary screening rule keeps only tolerances whose empty-target fraction on the fitting data is at most $\alpha$, and an empty retained grid selects no tolerance. The main tolerance grids are fixed before evaluation.

\paragraph{Certificate slack.}
With a fixed query trace and a normalizer that does not depend on $\Delta$, we split the gap between the threshold bracket and the true quantile into upper-bound and lower-bound slack, using set sizes and infinity indicators for infinite endpoints. We also count examples whose true target is nonempty but whose certified target is empty after every candidate is queried. Figure~\ref{fig:certificate-slack} shows the results. The observed-minimum upper bound \(U_x(M_x)\), which is unchanged across \(\Delta\) for a fixed query trace, certifies nonmembership; the supplied lower bound \(L_x=m_x-\Delta\) certifies acceptance.
%% [FRONTIER] end

\subsection{Implementation checks}
\label{app:exp-interface}
Before the reported runs, unit tests checked the acquisition and certificate code on all-negative, singleton, and zero-size pools, tied scores, unequal pool sizes, $k=n+1$, early certification of infinite thresholds, known finite quantiles and terminal-layer costs, $R_x=0$, and both TabRepo target formulas. They also check that no policy reads unqueried verdicts, that candidate pools match across rankings, that MATH problems are aligned, that $C-H=EV$ holds, that split boundaries are correct, and that bootstrap resampling keeps splits paired. Exhaustive simulations on small pools check every policy against an independent implementation. Every experiment group passed these checks in both the pilot and the reported runs.
%% [FRONTIER] begin
For the frontier, the tests also cover exact and incomplete certifiers, the normalizers, inverse plateaus, finite-top round trips, the zero conditions of the normalized gaps, set-size telescoping, and common query histories.
%% [FRONTIER] end

\subsection{Supplementary results}
\label{app:exp-results}
All 10,000 splits of each main setting have finite $q$, and the saved-tail identity of Appendix~\ref{app:ties} holds exactly on every split. The mean saved-tail and terminal-layer savings are 88.134 and 3.559 queries for TREC/BM25, 8,053.689 and 29.179 for MATH 1B/PRM, 6,868.067 and 29.419 for MATH 3B/PRM, 143.883 and 7.214 for TabRepo/consensus, and 18.624 and 5.284 for TabRepo/random. In each setting, their sum is the saving against Exact.

%% [FRONTIER] begin
\paragraph{Simultaneous frontiers.}
At $\alpha=0.1$ and $\varepsilon=0.35$ on TabRepo with the consensus ranking and the full normalized-regret grid, the ECDF supremum frontier has mean set size 606.843 over 10,000 splits, compared with 208.761 for the shared normalizer, 134.113 for Bonferroni, and 1.000 for pointwise calibration. On the first 2,000 of these splits, finite-top gives 6.823, against 208.779, 133.224, and 598.963 for the shared, Bonferroni, and ECDF supremum constructions. Its paired differences from these three are $-201.956$ (95\% CI $-207.686$ to $-195.783$), $-126.402$ ($-133.054$ to $-119.233$), and $-592.140$ ($-621.093$ to $-563.363$). Infinite thresholds, which return the full test pool, occur on 3,459 of the 10,000 ECDF splits and on 682 of the first 2,000. At the strictest tolerance, finite-top returns larger sets than the shared normalizer (Figure~\ref{fig:frontier}).

At the terminal budget, the split-level crossing frequency is 0.0019 for the ECDF and 0.0300 for finite-top on TabRepo, and 0.0092 and 0.0100 on TREC/BM25. Finite-score saturation, where some relevant finite score reaches the normalized top, occurs on 91.24\% of TabRepo ECDF splits and on no finite-top split. Under absolute loss, ECDF saturation occurs on every split, and the finite-top frontier is still infinite at the strictest tolerances because their targets are often empty.

On the relaxed normalized-regret grid under the consensus ranking, the ECDF supremum has joint coverage 0.9951 and is finite at every tolerance on 65.55\% of splits, while the shared normalizer has joint coverage 0.9738 and is always finite. Under the random ranking, the ECDF supremum has joint coverage 0.9510 and all-finite rate 0.8725 on the full grid, and 0.9606 and 0.9074 on the relaxed grid. These values use 10,000 splits at $\alpha=0.1$, and ECDF saturation also occurs on the relaxed grid.

The five-slot matched-family control uses the same 117 datasets, fitting split, ranking, and all-candidate trace, with 2,000 splits. Every fitting split yields five distinct levels, with mean achieved fractions within $2.1\times10^{-5}$ of their targets. Under normalized regret with the consensus ranking at $\alpha=0.1$, the ECDF supremum has joint coverage 0.9385 and all-finite rate 0.7225, and the shared normalizer has joint coverage 0.8920. The matched absolute-loss grid has joint literal success 0.1195 and never gives an all-finite supremum frontier. Matching the average target fractions leaves the empty-target probabilities of the two families different.
%% [FRONTIER] end

\paragraph{MATH timeout sensitivity.}
The four timeouts are source indices 43 and 240 for 1B seed 0, index 254 for 3B seed 3, and index 159 for 3B seed 4. We evaluate every binary assignment of their labels within each affected pool. In the 1B PRM main comparison, every assignment gives mean Exact, Serial, and Layer costs of 10,960.7311, 10,924.6289, and 2,877.8637, with literal and extended coverage both 0.8072 at $\alpha=0.2$. The 3B main comparison uses seed 0 and contains no timeout. Costs for smaller pools and random rankings vary across assignments, and the supplementary tables give their full ranges.

\begin{figure}[t]
\centering
\includegraphics[width=0.9\textwidth]{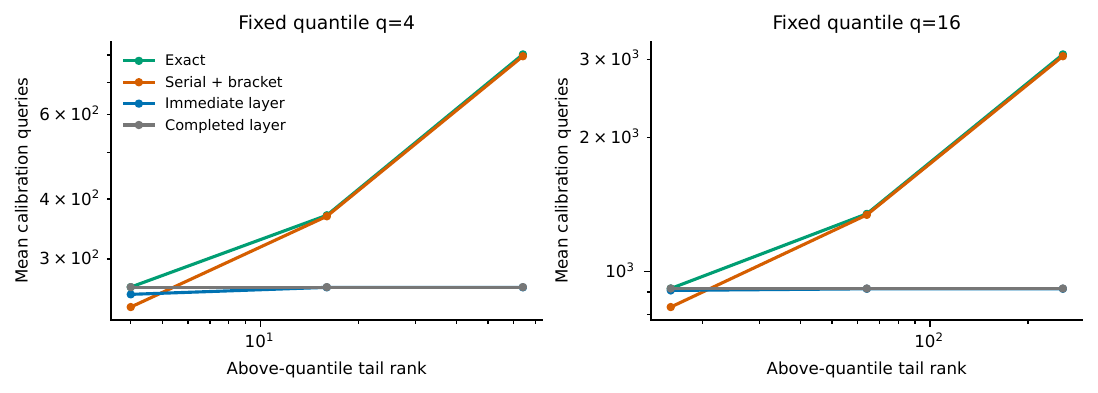}
\caption{Constructed saved tails at fixed quantile $q$, with 200 example orders per setting. The completed-layer saving is exactly $9(\tau-q)$. Serial is cheaper when $\tau=q$.}
\label{fig:saved-tails}
\end{figure}
%% [FRONTIER] begin
\begin{figure}[t]
\centering
\includegraphics[width=0.9\textwidth]{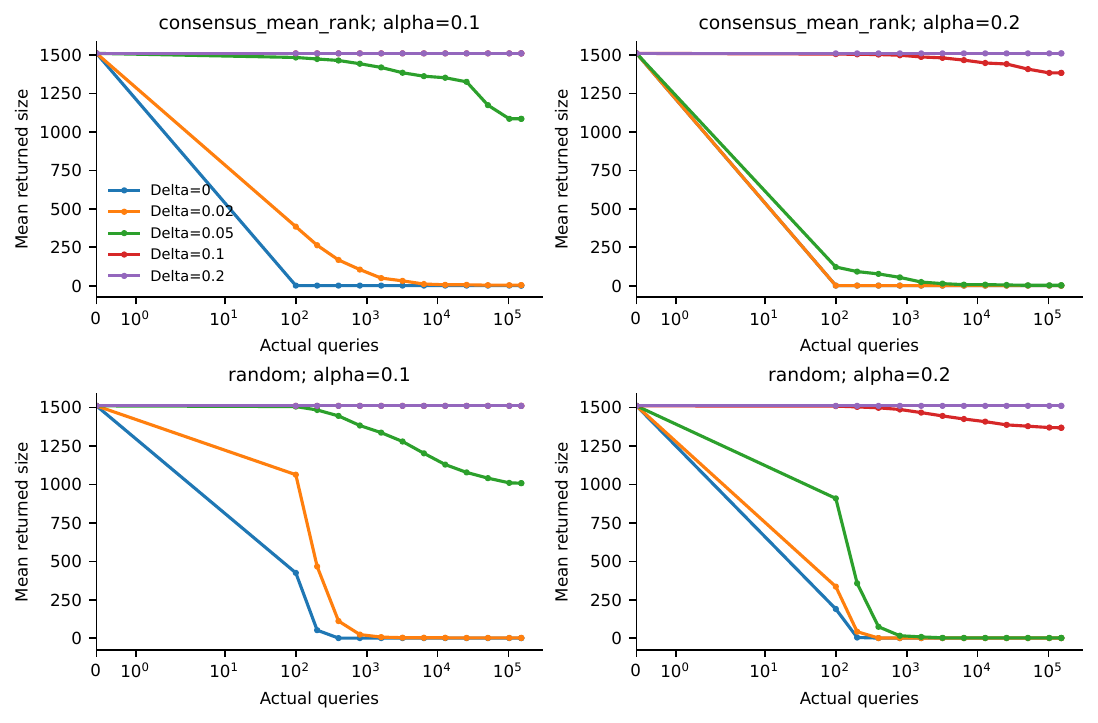}
\caption{Certificate slack under a common all-candidate query trace, with 2,000 paired splits per setting. The supplied lower bound \(L_x=m_x-\Delta\) varies with \(\Delta\); the observed-minimum upper bound \(U_x(M_x)\) is the same across \(\Delta\) at each point of the common trace. Infinite thresholds contribute the full test pool to mean size.}
\label{fig:certificate-slack}
\end{figure}
%% [FRONTIER] end

\paragraph{Generation seeds and pool sizes.}
Figure~\ref{fig:math-seeds} reports all five 3B generation seeds on aligned problems. At $K=256$ and $\alpha=0.2$, mean Layer costs under the PRM ranking range from 1,285.82 to 1,545.17 queries, with savings against Serial from 80.25\% to 83.24\%. Under the random ranking, mean Layer costs range from 1,508.92 to 1,610.33. Every split in these comparisons has a finite quantile. Figure~\ref{fig:math-pools} shows the pool-size controls. For 1B with the PRM ranking at $K=32$ and $\alpha=0.2$, all 2,000 splits have $q=\infty$, and the mean costs of Serial and Layer are 2,424.99 and 3,434.31.

\begin{figure}[t]
\centering
\includegraphics[width=0.9\textwidth]{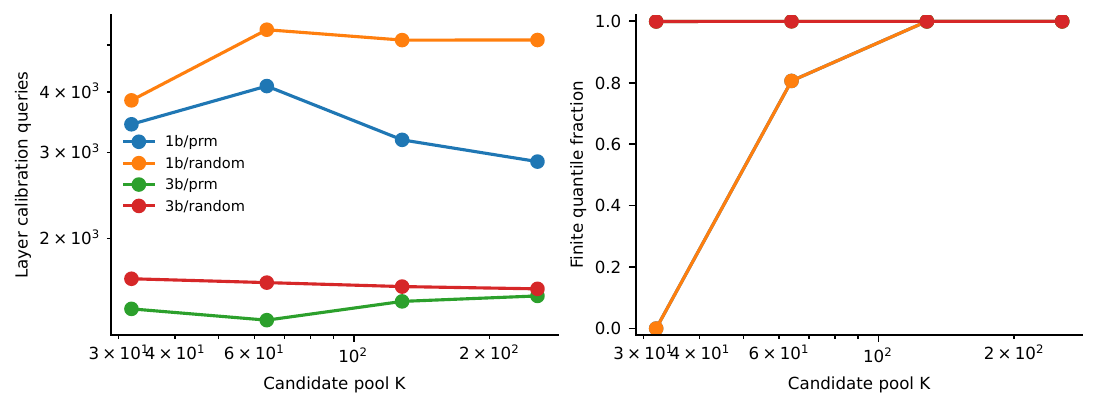}
\caption{MATH pool-size controls at $\alpha=0.2$. Each pool contains the first $K$ completions, selected before ranking. Mean costs include splits with infinite thresholds. The two PRM rows at $K=256$ use 10,000 splits and the other rows 2,000.}
\label{fig:math-pools}
\end{figure}
\begin{figure}[t]
\centering
\includegraphics[width=0.9\textwidth]{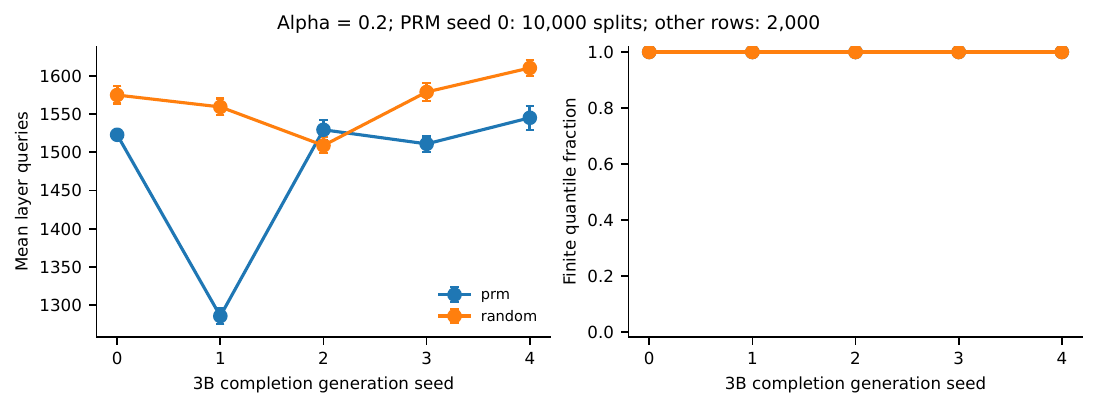}
\caption{MATH 3B with five generation seeds on the same 500 problems, at $K=256$ and $\alpha=0.2$. Each seed gives a separate candidate pool. Error bars are 95\% split-bootstrap intervals. PRM seed 0 uses the 10,000 splits of the main comparison and the other rows 2,000.}
\label{fig:math-seeds}
\end{figure}

\paragraph{Independent replay of the terminal layer.}
A separate simulator executes the completed terminal layer from the observed verdicts and ranks. Its cost and terminal rank agree with the formula-based values on all 133,590 finite-$q$ split and risk-level records of the 32 acquisition groups.

\paragraph{Total cost with deployment.}
All policies return the same predictor, so they incur the same setup and deployment costs. The absolute calibration saving is therefore the same at every deployment count, while its share of the total cost falls as deployments grow. With 1,000 deployments, the query savings against Serial are 1.94\% for TREC, 39.86\% for MATH 1B, 51.42\% for MATH 3B, 1.51\% for TabRepo/random, and 10.86\% for TabRepo/consensus. Fitting a new TabRepo consensus ranking costs 151,059.50 queries on average, which reduces the total query saving at one deployment to 0.088\%, compared with 58.93\% for calibration alone. Table~\ref{tab:revision-cost-ledger} gives the full ledger.

\paragraph{Recorded training time.}
All 368,576 available configuration--dataset pairs in TabRepo have a recorded fitting duration. Summing the three-fold durations along the queries of each policy gives, at $\alpha=0.1$, mean Serial and Layer costs of 380,482.84 and 182,179.35 seconds under the consensus ranking and 93,194.91 and 84,299.54 seconds under the random ranking. The reductions are 52.12\% (95\% CI 51.51--52.74\%) and 9.54\% (9.30--9.79\%) over 10,000 paired splits, compared with query savings of 58.93\% and 15.00\%. Appendix~\ref{app:revision-cost} describes the time ledger.

\subsection{Retrospective certificate analysis and boundary ties}
\label{app:complexity-analysis}
We evaluate the verification-complexity results under the information model of Section~\ref{sec:setup} on every split of the acquisition experiments. A group fixes the benchmark, the candidate pool (including the MATH generation seed and $K$), the ranking, and the target, and a record is one split at one risk level with the outcomes of Exact, Serial, and Layer. The 32 groups at two risk levels give 64 cells. Five main cells have 10,000 splits: TREC/BM25 and MATH 1B and 3B with seed-0 PRM ranking at $K=256$ and $\alpha=0.2$, and TabRepo with the consensus and random rankings at $\alpha=0.1$. The other 59 cells have 2,000 splits, for 168,000 records in total, 133,590 of them with finite $q$. The calibration-size and absolute-loss comparisons are analyzed separately.

Score vectors are rebuilt from the verdict tables with the split and ranking rules of Appendix~\ref{app:exp-splits}. Random rankings are regenerated from the same random streams, and TabRepo rankings are refitted on the fitting data or redrawn, with ranks recomputed among available configurations. The rebuilt $q$ and Exact cost agree with the acquisition runs on every split, and the completed-layer cost, saved tail, and terminal-layer saving agree on every finite-$q$ split. Full-information scores enter only this analysis, and the policies observe only the information described in Appendix~\ref{app:exp-policies}.

For each record we compute $B_q$, $N_q$, $h$, $J$, the conditional mean of Proposition~\ref{prop:final-layer} and the variance of Appendix~\ref{app:final-layer}, both components of $\Omega$, $\ell=\#\{i:\kappa_i<q\}$, $e=\#\{i:\kappa_i=q\}$, $W_{\mathrm f}$, $W$ from the sorting formula of Appendix~\ref{app:prefix-formula}, and the saved-tail and overhead terms. Infinite-$q$ certificates use the $m$ cheapest truly empty pools and are summarized separately. Every applicable finite-$q$ record satisfies the integer check $mT^+\leq(2m-1)W$, the tie-sensitive bound of Appendix~\ref{app:ties}, and the unique-boundary bound, and the two overhead terms sum to the observed overhead. Exhaustive enumeration of verdict tables and prefix depths on small pools, including empty and unequal pools and $k=n+1$, confirms the certificate minima.

Mean ratios average $T/W$ within splits, and ratios of mean costs are reported separately. Ratios with a zero denominator are excluded, and their counts are recorded. Unique-boundary summaries use the finite-$q$ splits with $e=1$, which number 4,734 for TREC, 6,306 for MATH 1B, and 2,956 for MATH 3B in the main settings. Random-query costs are available for 32,000 records and are compared with $W_{\mathrm f}$.

The boundary experiment crosses three quantiles $q\in\{2,16,64\}$ with six boundary multiplicities $e\in\{1,2,5,10,30,90\}$. Each table has $n=99$, $k=90$, and common pool size $q+8$, with $90-e$ first acceptable ranks equal to one, $e$ equal to $q$, and nine equal to $q+8$. Each pool has a single acceptable candidate at its specified rank, so the above-quantile tail mass is 72 in every setting. All 18 settings use the same 200 visitation permutations, generated by NumPy \texttt{SeedSequence([seed,91])} for seeds 0--199, with Serial and Layer paired within each permutation. Figure~\ref{fig:boundary-ties} reports all settings, and Table~\ref{tab:boundary-costs} gives their certificate and acquisition costs.

The analysis runs on CPUs with at most 16 processes and one numerical-library thread per process, on a machine with 96 logical CPUs and about 1~TiB of RAM. It uses the cached completions, PRM scores, and verdicts. The source package includes the certificate implementation, tests, analysis scripts, and per-split and summary tables.

\begin{table}[t]
\centering\small
\setlength{\tabcolsep}{4pt}
\caption{Boundary-tie costs in queries, using 200 common visitation permutations for every cell and fixed above-quantile tail mass 72. Overhead is the mean $\Omega=T-W$; Serial is the mean $T_s$ with the same bracket.}
\label{tab:boundary-costs}
\begin{tabular}{@{}r rrr rrr rrr@{}}
\toprule
& \multicolumn{3}{c}{$q=2$} & \multicolumn{3}{c}{$q=16$} & \multicolumn{3}{c}{$q=64$}\\
\cmidrule(lr){2-4}\cmidrule(lr){5-7}\cmidrule(lr){8-10}
$e$ & $W$ & Overhead & Serial & $W$ & Overhead & Serial & $W$ & Overhead & Serial\\
\midrule
1 & 100.00 & 4.55 & 180.10 & 240.00 & 4.55 & 320.10 & 720.00 & 4.55 & 800.10\\
2 & 100.00 & 7.04 & 180.91 & 240.00 & 21.05 & 333.51 & 720.00 & 69.05 & 856.71\\
5 & 100.00 & 11.43 & 183.90 & 240.00 & 67.42 & 378.36 & 720.00 & 259.43 & 1045.08\\
10 & 100.00 & 17.28 & 188.90 & 240.00 & 143.28 & 453.36 & 720.00 & 575.28 & 1360.08\\
30 & 120.00 & 17.73 & 208.90 & 540.00 & 143.72 & 753.36 & 1980.00 & 575.73 & 2620.08\\
90 & 180.00 & 17.89 & 268.90 & 1440.00 & 143.89 & 1653.36 & 5760.00 & 575.89 & 6400.08\\
\bottomrule
\end{tabular}
\end{table}

For each constructed table, the empirical variance is the unbiased sample variance over its 200 visitation permutations, and the analytical variance is conditional on that table. The same 2,000 bootstrap draws of permutation indices are used in every $(q,e)$ cell. Figure~\ref{fig:boundary-ties} compares analytical and observed means, and Table~\ref{tab:boundary-variance} reports the variances.

\begin{table}[t]
\centering\small
\caption{Conditional scheduling variance of $T$ (queries squared) on each fixed synthetic table. Empirical variance uses 200 archived visitation permutations and the unbiased sample formula. Intervals are 95\% percentile intervals from 2,000 permutation bootstrap resamples, using the same resampled permutation IDs in all 18 conditions. They describe scheduling Monte Carlo uncertainty conditional on each fixed table.}
\label{tab:boundary-variance}
\begin{tabular}{@{}rrrrc@{}}
\toprule
$q$ & $e$ & Theory & Empirical & 95\% interval\\
\midrule
2 & 1 & 8.250 & 7.384 & [6.339, 8.352]\\
2 & 2 & 6.000 & 6.184 & [5.250, 7.160]\\
2 & 5 & 2.679 & 2.999 & [2.461, 3.534]\\
2 & 10 & 1.240 & 1.338 & [0.930, 1.746]\\
2 & 30 & 0.351 & 0.371 & [0.226, 0.564]\\
2 & 90 & 0.106 & 0.118 & [0.065, 0.182]\\
16 & 1 & 8.250 & 7.384 & [6.339, 8.352]\\
16 & 2 & 6.000 & 6.184 & [5.250, 7.160]\\
16 & 5 & 2.679 & 2.999 & [2.461, 3.534]\\
16 & 10 & 1.240 & 1.338 & [0.930, 1.746]\\
16 & 30 & 0.351 & 0.371 & [0.226, 0.564]\\
16 & 90 & 0.106 & 0.118 & [0.065, 0.182]\\
64 & 1 & 8.250 & 7.384 & [6.339, 8.352]\\
64 & 2 & 6.000 & 6.184 & [5.250, 7.160]\\
64 & 5 & 2.679 & 2.999 & [2.461, 3.534]\\
64 & 10 & 1.240 & 1.338 & [0.930, 1.746]\\
64 & 30 & 0.351 & 0.371 & [0.226, 0.564]\\
64 & 90 & 0.106 & 0.118 & [0.065, 0.182]\\
\bottomrule
\end{tabular}
\end{table}

\subsection{Paired calibration-sample comparisons}
\label{app:revision-size}
We run Exact, Serial, and Layer on the 58 size, risk, and setting cells of Appendix~\ref{app:exp-cost}, with seeds 1000--2999. Within a split, every size shares the fitting data, ranking, and test example. A sample of size $n'$ consists of the first $n'$ examples in the visitation order of the split, which are also queried in that order. Each policy starts with no queried verdicts and uses the index $k'$ of its sample. For the held-out example we record set size, literal hit, extended coverage, the finite-threshold and full-pool indicators, and the deployment cost. TabRepo uses the relative target of Appendix~\ref{app:exp-targets}.

At full size, the costs and outputs of all three policies coincide with those of the matched-threshold comparison on the same 24,000 setting, risk, and split records. At every size, the three policies return the same thresholds and prediction sets.

Table~\ref{tab:revision-calibration-size} gives mean query costs, set sizes with their standard deviations and empirical 90th percentiles, finite-threshold rates, and literal hit rates. Table~\ref{tab:final-week-size-pairs} reports the differences between small-sample and full-sample Layer. Figures~\ref{fig:sample-size-all} and~\ref{fig:sample-size-all2} show every size at both risk levels for the six settings. Bootstrap resamples are shared across policies, sizes, risk levels, and endpoints.

On MATH 3B at $\alpha=0.2$, reducing the calibration size from 300 to 75 lowers the Layer cost from 1,536.35 to 490.08 queries, while the mean set size rises from 17.23 to 23.06 and its standard deviation from 4.20 to 17.08. On MATH 1B, the corresponding costs are 2,889.99 and 904.45 and the standard deviations 7.83 and 27.98. On MATH, larger calibration samples cost more and give less variable set sizes. At the smallest size for $\alpha=0.2$, $n'=4$, Layer costs more than Serial on both MATH pools because some reference thresholds are infinite. Under the TabRepo consensus ranking, mean set sizes are close to one: size 50 gives 1.001 with 46.90 Layer queries, and size 100 gives 1.000 with 92.79. The nonmonotone curves come from changes in the discrete index $k'$.

\begingroup\small\setlength{\tabcolsep}{3pt}
\begin{longtable}{@{}rrrrrrrrrr@{}}
\caption{All 58 policy-by-size comparisons on 2,000 common splits per cell. Exact, Serial and Layer have identical thresholds and sets at each size. Costs and size are unconditional means; SD and P90 describe the size distribution. Finite and Hit are rates. Complete intervals and spending quantiles are in the supplied CSV.}\label{tab:revision-calibration-size}\\
\toprule
$n'$ & Exact & Serial & Layer & Size & SD & P90 & Finite (a) & Hit & Saving vs. Serial\\
\midrule\endfirsthead
\toprule
$n'$ & Exact & Serial & Layer & Size & SD & P90 & Finite (a) & Hit & Saving vs. Serial\\
\midrule\endhead
\bottomrule\endfoot
\multicolumn{10}{@{}l}{TREC / BM25, $\alpha=0.1$}\\
9 & 58.72 & 54.67 & 58.70 & 33.200 & 43.04 & 90 & 0.8265 & 0.8895 & -7.37\%\\
10 & 65.55 & 60.30 & 65.52 & 35.626 & 44.56 & 103 & 0.8060 & 0.9025 & -8.67\%\\
20 & 130.91 & 128.68 & 111.98 & 20.943 & 18.10 & 30 & 0.9655 & 0.9030 & 12.98\%\\
40 & 263.70 & 261.99 & 209.00 & 16.824 & 3.78 & 19 & 1.0000 & 0.8985 & 20.23\%\\
\addlinespace
\multicolumn{10}{@{}l}{TREC / BM25, $\alpha=0.2$}\\
4 & 26.03 & 25.37 & 26.03 & 19.257 & 31.19 & 31 & 0.9255 & 0.8035 & -2.59\%\\
5 & 32.40 & 31.36 & 32.40 & 22.267 & 34.30 & 31 & 0.9080 & 0.8340 & -3.29\%\\
10 & 65.55 & 63.59 & 50.51 & 13.015 & 11.14 & 19 & 0.9905 & 0.8090 & 20.58\%\\
20 & 130.91 & 128.26 & 90.10 & 10.915 & 3.72 & 16 & 1.0000 & 0.8025 & 29.75\%\\
40 & 263.70 & 259.13 & 172.75 & 10.431 & 2.26 & 13 & 1.0000 & 0.7990 & 33.34\%\\
\addlinespace
\multicolumn{10}{@{}l}{TREC / document ID, $\alpha=0.1$}\\
9 & 148.79 & 138.06 & 146.51 & 64.745 & 41.39 & 111 & 0.8265 & 0.8925 & -6.12\%\\
10 & 165.99 & 151.90 & 163.03 & 67.771 & 41.47 & 116 & 0.8060 & 0.9045 & -7.33\%\\
20 & 332.32 & 327.41 & 311.56 & 62.693 & 26.72 & 92 & 0.9655 & 0.8975 & 4.84\%\\
40 & 666.12 & 663.81 & 601.99 & 60.133 & 17.75 & 79 & 1.0000 & 0.8950 & 9.31\%\\
\addlinespace
\multicolumn{10}{@{}l}{TREC / document ID, $\alpha=0.2$}\\
4 & 66.15 & 64.72 & 65.86 & 42.632 & 37.85 & 83 & 0.9255 & 0.8035 & -1.77\%\\
5 & 81.95 & 79.49 & 81.35 & 47.958 & 39.19 & 87 & 0.9080 & 0.8310 & -2.35\%\\
10 & 165.99 & 162.52 & 137.57 & 38.351 & 26.83 & 76 & 0.9905 & 0.8135 & 15.35\%\\
20 & 332.32 & 326.38 & 231.65 & 30.992 & 18.58 & 65 & 1.0000 & 0.8045 & 29.03\%\\
40 & 666.12 & 657.72 & 410.14 & 26.540 & 12.17 & 44 & 1.0000 & 0.8010 & 37.64\%\\
\addlinespace
\multicolumn{10}{@{}l}{MATH 1B / PRM, $\alpha=0.1$}\\
9 & 328.39 & 228.95 & 328.13 & 181.830 & 96.81 & 256 & 0.3995 & 0.8840 & -43.32\%\\
38 & 1,395.97 & 988.27 & 1,336.67 & 219.887 & 66.87 & 256 & 0.2510 & 0.9015 & -35.25\%\\
75 & 2,754.38 & 2,342.11 & 2,497.03 & 205.112 & 69.78 & 256 & 0.3795 & 0.8995 & -6.61\%\\
150 & 5,495.91 & 5,112.10 & 4,736.05 & 194.528 & 63.96 & 256 & 0.5185 & 0.8990 & 7.36\%\\
300 & 10,975.33 & 10,656.32 & 9,421.45 & 198.487 & 52.78 & 256 & 0.5705 & 0.9010 & 11.59\%\\
\addlinespace
\multicolumn{10}{@{}l}{MATH 1B / PRM, $\alpha=0.2$}\\
4 & 143.82 & 125.04 & 143.78 & 110.593 & 108.00 & 256 & 0.6755 & 0.8000 & -14.98\%\\
38 & 1,395.97 & 1,364.83 & 668.70 & 67.924 & 60.87 & 142 & 0.9435 & 0.8335 & 51.00\%\\
75 & 2,754.38 & 2,727.74 & 904.45 & 41.624 & 27.98 & 70 & 0.9975 & 0.8180 & 66.84\%\\
150 & 5,495.91 & 5,464.22 & 1,579.05 & 35.434 & 15.31 & 60 & 1.0000 & 0.8155 & 71.10\%\\
300 & 10,975.33 & 10,938.56 & 2,889.99 & 31.570 & 7.83 & 42 & 1.0000 & 0.8165 & 73.58\%\\
\addlinespace
\multicolumn{10}{@{}l}{MATH 3B / PRM, $\alpha=0.1$}\\
9 & 253.55 & 199.72 & 253.43 & 156.446 & 102.00 & 256 & 0.5435 & 0.8895 & -26.89\%\\
38 & 1,074.85 & 912.73 & 956.63 & 176.978 & 83.08 & 256 & 0.5515 & 0.9150 & -4.81\%\\
75 & 2,129.96 & 2,051.00 & 1,621.14 & 143.070 & 72.62 & 256 & 0.7925 & 0.9055 & 20.96\%\\
150 & 4,244.24 & 4,219.64 & 2,736.27 & 114.237 & 45.54 & 173 & 0.9650 & 0.9045 & 35.15\%\\
300 & 8,451.33 & 8,438.97 & 5,085.11 & 104.516 & 16.85 & 121 & 1.0000 & 0.9005 & 39.74\%\\
\addlinespace
\multicolumn{10}{@{}l}{MATH 3B / PRM, $\alpha=0.2$}\\
4 & 109.97 & 100.40 & 109.95 & 88.983 & 102.30 & 256 & 0.7755 & 0.8085 & -9.51\%\\
38 & 1,074.85 & 1,056.20 & 380.59 & 39.587 & 39.77 & 89 & 0.9910 & 0.8185 & 63.97\%\\
75 & 2,129.96 & 2,105.33 & 490.08 & 23.056 & 17.08 & 40 & 1.0000 & 0.8055 & 76.72\%\\
150 & 4,244.24 & 4,210.58 & 839.43 & 19.204 & 8.86 & 31 & 1.0000 & 0.8055 & 80.06\%\\
300 & 8,451.33 & 8,417.97 & 1,536.35 & 17.226 & 4.20 & 25 & 1.0000 & 0.8015 & 81.75\%\\
\addlinespace
\multicolumn{10}{@{}l}{TabRepo / random, $\alpha=0.1$}\\
9 & 11.39 & 11.39 & 11.39 & 2.962 & 5.97 & 5 & 1.0000 & 0.9645 & 0.00\%\\
13 & 16.33 & 16.33 & 16.33 & 3.590 & 6.72 & 7 & 1.0000 & 0.9720 & 0.00\%\\
25 & 31.48 & 30.89 & 27.27 & 1.848 & 1.61 & 3 & 1.0000 & 0.9650 & 11.72\%\\
50 & 62.75 & 60.07 & 51.19 & 1.381 & 0.61 & 2 & 1.0000 & 0.9540 & 14.79\%\\
100 & 125.14 & 119.10 & 100.96 & 1.319 & 0.55 & 2 & 1.0000 & 0.9550 & 15.23\%\\
\addlinespace
\multicolumn{10}{@{}l}{TabRepo / random, $\alpha=0.2$}\\
4 & 5.08 & 5.08 & 5.08 & 1.972 & 4.34 & 3 & 1.0000 & 0.9450 & 0.00\%\\
13 & 16.33 & 15.46 & 13.48 & 1.392 & 0.85 & 2 & 1.0000 & 0.9500 & 12.79\%\\
25 & 31.48 & 28.11 & 23.32 & 1.155 & 0.39 & 2 & 1.0000 & 0.9420 & 17.03\%\\
50 & 62.75 & 54.86 & 45.47 & 1.133 & 0.35 & 2 & 1.0000 & 0.9370 & 17.12\%\\
100 & 125.14 & 108.43 & 89.75 & 1.128 & 0.34 & 2 & 1.0000 & 0.9390 & 17.23\%\\
\addlinespace
\multicolumn{10}{@{}l}{TabRepo / consensus, $\alpha=0.1$}\\
9 & 21.11 & 21.11 & 21.11 & 13.041 & 56.84 & 6 & 1.0000 & 0.9800 & 0.00\%\\
13 & 28.78 & 28.78 & 28.78 & 16.683 & 65.29 & 16 & 1.0000 & 0.9805 & 0.00\%\\
25 & 58.54 & 56.40 & 25.32 & 1.489 & 3.74 & 1 & 1.0000 & 0.9790 & 55.10\%\\
50 & 121.23 & 113.35 & 46.90 & 1.001 & 0.03 & 1 & 1.0000 & 0.9785 & 58.62\%\\
100 & 242.19 & 226.24 & 92.79 & 1.000 & 0.00 & 1 & 1.0000 & 0.9785 & 58.99\%\\
\addlinespace
\multicolumn{10}{@{}l}{TabRepo / consensus, $\alpha=0.2$}\\
4 & 9.32 & 9.32 & 9.32 & 6.309 & 38.36 & 1 & 1.0000 & 0.9790 & 0.00\%\\
13 & 28.78 & 27.02 & 12.40 & 1.097 & 1.31 & 1 & 1.0000 & 0.9785 & 54.12\%\\
25 & 58.54 & 50.29 & 21.40 & 1.000 & 0.00 & 1 & 1.0000 & 0.9785 & 57.44\%\\
50 & 121.23 & 99.00 & 41.80 & 1.000 & 0.00 & 1 & 1.0000 & 0.9785 & 57.78\%\\
100 & 242.19 & 201.20 & 82.59 & 1.000 & 0.00 & 1 & 1.0000 & 0.9785 & 58.95\%\\
\addlinespace
\end{longtable}
\par\smallskip\noindent(a) Finite is the fraction of splits with a finite threshold $q$; infinite thresholds return the full available pool.
\endgroup

\begingroup\small\setlength{\tabcolsep}{4pt}
\begin{longtable}{@{}rrrrr@{}}
\caption{Small-sample Layer minus full-sample Layer, with pointwise 95\% paired split-bootstrap intervals. Negative cost differences favor the smaller sample. Positive size differences mean that it returns larger sets. Same-$q$ includes jointly infinite thresholds.}\label{tab:final-week-size-pairs}\\
\toprule
$n'$ & $\Delta$ queries [95\% CI] & $\Delta$ size [95\% CI] & Same $q$ & Finite pairs (a)\\
\midrule\endfirsthead
\toprule
$n'$ & $\Delta$ queries [95\% CI] & $\Delta$ size [95\% CI] & Same $q$ & Finite pairs (a)\\
\midrule\endhead
\bottomrule\endfoot
\multicolumn{5}{@{}l}{TREC / BM25, $\alpha=0.1$; full $n=40$: Layer 209.00, size 16.824}\\
9 & -150.30 [-151.92,-148.65] & 16.377 [14.591,18.323] & 0.1350 & 1653/2000\\
10 & -143.47 [-145.09,-141.79] & 18.802 [16.956,20.837] & 0.1435 & 1612/2000\\
20 & -97.02 [-98.48,-95.56] & 4.119 [3.409,4.960] & 0.2615 & 1931/2000\\
40 & 0.00 [0.00,0.00] & 0.000 [0.000,0.000] & 1.0000 & 2000/2000\\
\addlinespace
\multicolumn{5}{@{}l}{TREC / BM25, $\alpha=0.2$; full $n=40$: Layer 172.75, size 10.431}\\
4 & -146.72 [-148.23,-145.18] & 8.826 [7.564,10.251] & 0.1045 & 1851/2000\\
5 & -140.35 [-141.87,-138.78] & 11.836 [10.409,13.431] & 0.1100 & 1816/2000\\
10 & -122.24 [-123.71,-120.77] & 2.584 [2.129,3.072] & 0.1595 & 1981/2000\\
20 & -82.65 [-83.84,-81.51] & 0.484 [0.347,0.617] & 0.2745 & 2000/2000\\
40 & 0.00 [0.00,0.00] & 0.000 [0.000,0.000] & 1.0000 & 2000/2000\\
\addlinespace
\multicolumn{5}{@{}l}{TREC / document ID, $\alpha=0.1$; full $n=40$: Layer 601.99, size 60.133}\\
9 & -455.48 [-461.04,-449.94] & 4.612 [2.928,6.324] & 0.1180 & 1653/2000\\
10 & -438.96 [-444.49,-433.52] & 7.638 [5.898,9.397] & 0.1190 & 1612/2000\\
20 & -290.43 [-295.46,-285.74] & 2.561 [1.605,3.484] & 0.2115 & 1931/2000\\
40 & 0.00 [0.00,0.00] & 0.000 [0.000,0.000] & 1.0000 & 2000/2000\\
\addlinespace
\multicolumn{5}{@{}l}{TREC / document ID, $\alpha=0.2$; full $n=40$: Layer 410.14, size 26.540}\\
4 & -344.28 [-350.12,-338.55] & 16.093 [14.450,17.726] & 0.0790 & 1851/2000\\
5 & -328.79 [-334.64,-323.03] & 21.419 [19.696,23.077] & 0.0770 & 1816/2000\\
10 & -272.57 [-278.31,-267.08] & 11.812 [10.656,12.970] & 0.1105 & 1981/2000\\
20 & -178.50 [-183.66,-173.28] & 4.452 [3.743,5.204] & 0.1860 & 2000/2000\\
40 & 0.00 [0.00,0.00] & 0.000 [0.000,0.000] & 1.0000 & 2000/2000\\
\addlinespace
\multicolumn{5}{@{}l}{MATH 1B / PRM, $\alpha=0.1$; full $n=300$: Layer 9,421.45, size 198.487}\\
9 & -9,093.32 [-9,186.34,-8,995.83] & -16.657 [-21.036,-12.006] & 0.2775 & 482/2000\\
38 & -8,084.78 [-8,174.36,-7,989.31] & 21.400 [17.821,24.989] & 0.3555 & 327/2000\\
75 & -6,924.42 [-7,014.65,-6,833.19] & 6.625 [3.166,9.748] & 0.3360 & 524/2000\\
150 & -4,685.41 [-4,768.86,-4,604.72] & -3.958 [-6.601,-1.295] & 0.3375 & 772/2000\\
300 & 0.00 [0.00,0.00] & 0.000 [0.000,0.000] & 1.0000 & 1141/2000\\
\addlinespace
\multicolumn{5}{@{}l}{MATH 1B / PRM, $\alpha=0.2$; full $n=300$: Layer 2,889.99, size 31.570}\\
4 & -2,746.21 [-2,772.70,-2,720.61] & 79.023 [74.348,83.745] & 0.0100 & 1351/2000\\
38 & -2,221.29 [-2,252.36,-2,190.04] & 36.355 [33.725,39.084] & 0.0270 & 1887/2000\\
75 & -1,985.53 [-2,015.67,-1,958.04] & 10.055 [8.848,11.233] & 0.0575 & 1995/2000\\
150 & -1,310.94 [-1,335.94,-1,286.08] & 3.865 [3.283,4.445] & 0.0850 & 2000/2000\\
300 & 0.00 [0.00,0.00] & 0.000 [0.000,0.000] & 1.0000 & 2000/2000\\
\addlinespace
\multicolumn{5}{@{}l}{MATH 3B / PRM, $\alpha=0.1$; full $n=300$: Layer 5,085.11, size 104.516}\\
9 & -4,831.69 [-4,866.95,-4,796.48] & 51.929 [47.653,56.451] & 0.0120 & 1087/2000\\
38 & -4,128.48 [-4,165.14,-4,091.12] & 72.462 [68.904,76.014] & 0.0245 & 1103/2000\\
75 & -3,463.97 [-3,505.59,-3,422.62] & 38.553 [35.446,41.597] & 0.0445 & 1585/2000\\
150 & -2,348.84 [-2,387.28,-2,307.23] & 9.720 [7.940,11.536] & 0.0825 & 1930/2000\\
300 & 0.00 [0.00,0.00] & 0.000 [0.000,0.000] & 1.0000 & 2000/2000\\
\addlinespace
\multicolumn{5}{@{}l}{MATH 3B / PRM, $\alpha=0.2$; full $n=300$: Layer 1,536.35, size 17.226}\\
4 & -1,426.40 [-1,440.98,-1,411.66] & 71.757 [67.753,75.912] & 0.0145 & 1551/2000\\
38 & -1,155.76 [-1,173.94,-1,135.68] & 22.361 [20.710,24.097] & 0.0360 & 1982/2000\\
75 & -1,046.28 [-1,062.05,-1,030.99] & 5.831 [5.125,6.513] & 0.0670 & 2000/2000\\
150 & -696.93 [-710.42,-683.16] & 1.978 [1.656,2.316] & 0.1190 & 2000/2000\\
300 & 0.00 [0.00,0.00] & 0.000 [0.000,0.000] & 1.0000 & 2000/2000\\
\addlinespace
\multicolumn{5}{@{}l}{TabRepo / random, $\alpha=0.1$; full $n=100$: Layer 100.96, size 1.319}\\
9 & -89.57 [-90.12,-89.05] & 1.643 [1.381,1.918] & 0.6210 & 2000/2000\\
13 & -84.63 [-85.17,-84.09] & 2.271 [1.982,2.572] & 0.5475 & 2000/2000\\
25 & -73.69 [-74.12,-73.30] & 0.529 [0.462,0.602] & 0.6760 & 2000/2000\\
50 & -49.77 [-50.05,-49.51] & 0.062 [0.047,0.077] & 0.8950 & 2000/2000\\
100 & 0.00 [0.00,0.00] & 0.000 [0.000,0.000] & 1.0000 & 2000/2000\\
\addlinespace
\multicolumn{5}{@{}l}{TabRepo / random, $\alpha=0.2$; full $n=100$: Layer 89.75, size 1.128}\\
4 & -84.66 [-85.17,-84.17] & 0.844 [0.670,1.045] & 0.7405 & 2000/2000\\
13 & -76.27 [-76.68,-75.87] & 0.264 [0.232,0.301] & 0.7995 & 2000/2000\\
25 & -66.42 [-66.80,-66.08] & 0.028 [0.015,0.040] & 0.9430 & 2000/2000\\
50 & -44.27 [-44.53,-44.04] & 0.005 [-0.002,0.013] & 0.9715 & 2000/2000\\
100 & 0.00 [0.00,0.00] & 0.000 [0.000,0.000] & 1.0000 & 2000/2000\\
\addlinespace
\multicolumn{5}{@{}l}{TabRepo / consensus, $\alpha=0.1$; full $n=100$: Layer 92.79, size 1.000}\\
9 & -71.68 [-74.17,-69.19] & 12.041 [9.552,14.478] & 0.8290 & 2000/2000\\
13 & -64.01 [-66.88,-61.07] & 15.683 [12.781,18.625] & 0.7700 & 2000/2000\\
25 & -67.47 [-67.77,-67.11] & 0.489 [0.337,0.667] & 0.9290 & 2000/2000\\
50 & -45.88 [-45.92,-45.84] & 0.001 [0.000,0.003] & 0.9990 & 2000/2000\\
100 & 0.00 [0.00,0.00] & 0.000 [0.000,0.000] & 1.0000 & 2000/2000\\
\addlinespace
\multicolumn{5}{@{}l}{TabRepo / consensus, $\alpha=0.2$; full $n=100$: Layer 82.59, size 1.000}\\
4 & -73.28 [-74.91,-71.58] & 5.309 [3.674,6.997] & 0.9220 & 2000/2000\\
13 & -70.19 [-70.30,-70.05] & 0.097 [0.046,0.168] & 0.9805 & 2000/2000\\
25 & -61.19 [-61.24,-61.15] & 0.000 [0.000,0.000] & 1.0000 & 2000/2000\\
50 & -40.80 [-40.83,-40.76] & 0.000 [0.000,0.000] & 1.0000 & 2000/2000\\
100 & 0.00 [0.00,0.00] & 0.000 [0.000,0.000] & 1.0000 & 2000/2000\\
\addlinespace
\end{longtable}
\par\smallskip\noindent(a) Finite pairs counts splits with finite thresholds at both sample sizes, out of 2,000; jointly infinite pairs are included in Same-$q$.
\endgroup

\begin{figure}[p]
\centering
\includegraphics[width=\textwidth]{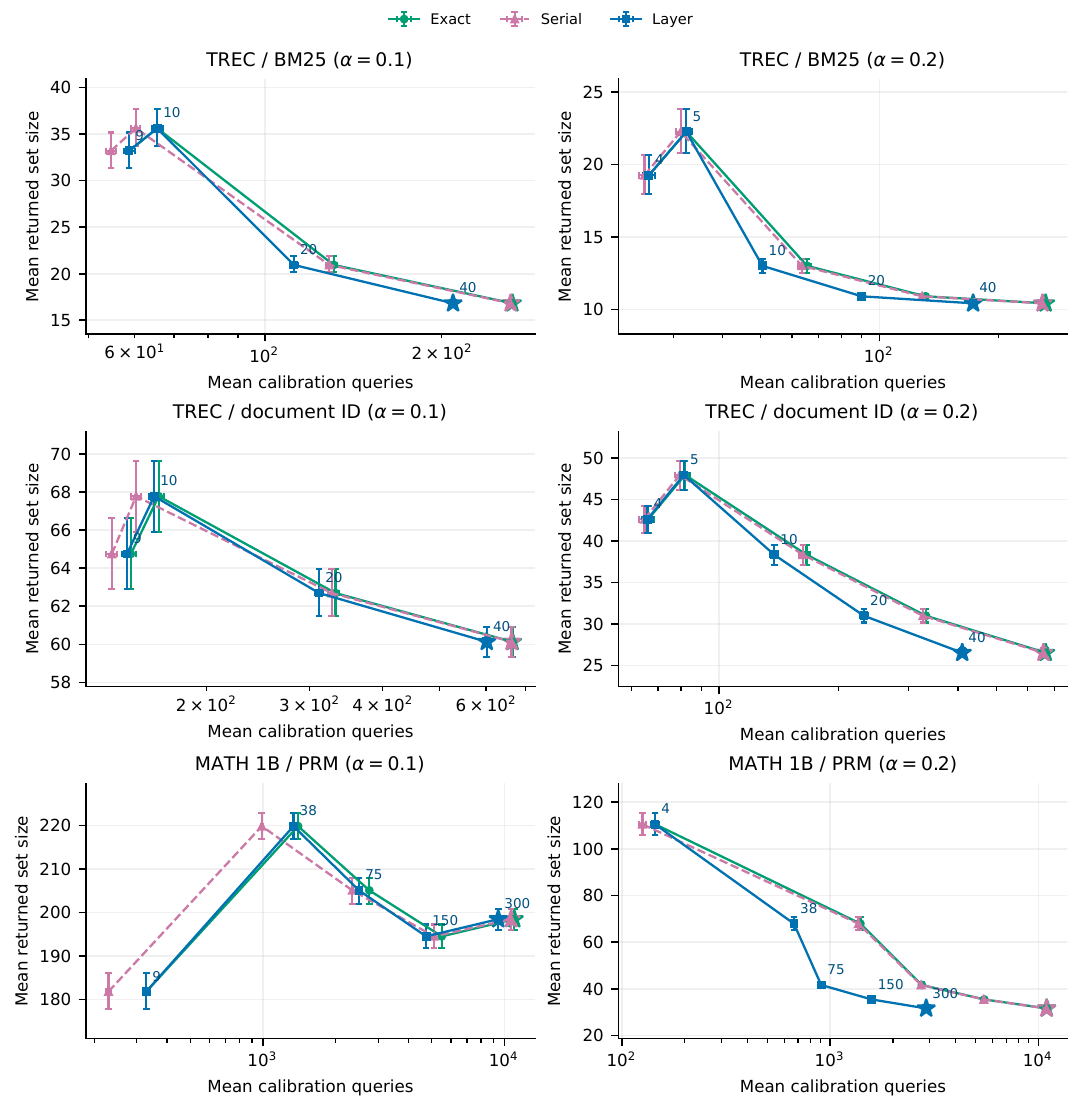}
\caption{Policy-by-size comparisons for TREC and MATH 1B at both risk levels, with 2,000 matched splits per cell. At each size, Exact, Serial, and Layer return the same predictor. Labels give Layer sample sizes, stars mark the full size, and error bars are pointwise 95\% split-bootstrap intervals for mean cost and size.}
\label{fig:sample-size-all}
\end{figure}
\begin{figure}[p]
\centering
\includegraphics[width=\textwidth]{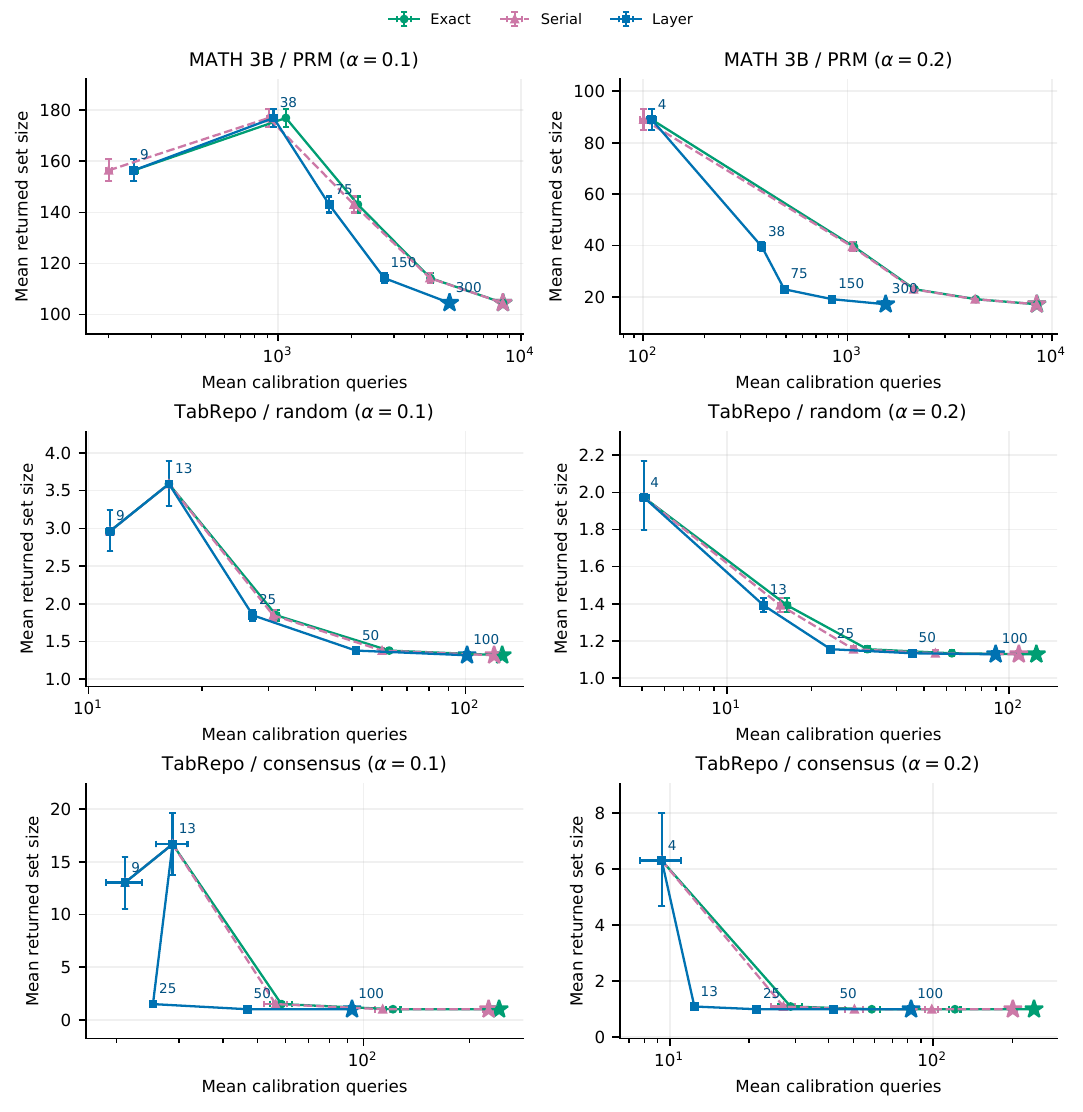}
\caption{Policy-by-size comparisons for MATH 3B and TabRepo, continuing Figure~\ref{fig:sample-size-all}.}
\label{fig:sample-size-all2}
\end{figure}

\subsection{Single-target absolute-loss acquisition}
\label{app:final-week-absolute}
This comparison uses the 117 ROC-AUC tasks and the absolute-loss target $J_x(a)\leq\varepsilon$ of Appendix~\ref{app:exp-targets}, with a 54/54/9 split and a consensus ranking fitted on the 54 fitting tasks. It uses seeds 1000--2999 with the split, visitation, and test-selection rules of Appendix~\ref{app:exp-splits}. At $\alpha=0.1$, $k=50$ and $m=5$. Each of the seven tolerances in $\calE_{\rm abs}$ defines a separate reference predictor, and each policy starts with no queried verdicts at each tolerance. A query reveals one validation loss and therefore membership, and costs count the queries of each policy. The three policies return identical thresholds and prediction sets in every cell.
%% [FRONTIER] begin
These thresholds also coincide with the terminal pointwise thresholds of the frontier experiment on the same 2,000 splits, whose costs differ because that experiment queries along a single history for the strictest target.
%% [FRONTIER] end

Table~\ref{tab:final-week-absolute} reports every tolerance. At $0.05$ and $0.10$, all reference thresholds are infinite. At $0.15$ and $0.20$, only 0.10\% and 0.90\% are finite, and Layer costs more than Serial. At $0.30$, 92.40\% are finite, the mean set size is 224.51, the literal hit rate is 0.9205, and Layer saves 73.44\% [71.56, 75.32] of the queries of Serial. At $0.40$, every threshold equals one, the literal hit rate is 0.9760, and the saving is 96.40\% [96.28, 96.52]. At $0.50$, every available candidate is acceptable, and Serial and Layer both use exactly 50 queries against 54 for Exact. With an infinite threshold, score coverage is one, but the literal hit rate can be much lower, 0.4790 at $\varepsilon=0.05$, because many test tasks have no acceptable configuration.

Durations are summed over the queries of each policy. At $0.40$, the 96.40\% query saving corresponds to a 65.07\% saving in recorded training time. Fitting the consensus ranking on the 54 fitting tasks costs 81,945.87 queries and 38,310,578 recorded seconds on average. Table~\ref{tab:final-week-absolute-ledger} adds deployment and ranking-fitting costs, with units defined in Appendix~\ref{app:revision-cost}. Bootstrap resamples are shared across all tolerances and policies.

\begin{table}[htbp]\centering\small\setlength{\tabcolsep}{3pt}
\caption{Direct absolute-loss TabRepo acquisition: 117 ROC-AUC tasks, 54 fitting and 54 calibration tasks, $\alpha=.1$, 2,000 paired splits per tolerance. All three policies return identical sets. Costs include infinite thresholds. Query/time savings are relative to Serial, with pointwise 95\% paired intervals; time sums source-recorded three-fold training durations.}
\label{tab:final-week-absolute}
\begin{tabular}{@{}rrrrrrrrr@{}}
\toprule
$\varepsilon$ & Finite (a) & Size & Hit & Exact & Serial & Layer & Query saving & Time saving\\
\midrule
0.05 & 0.0000 & 1,517.19 & 0.4790 & 42,157.0 & 7,596.4 & 41,267.1 & -443.25 [-445.54,-440.97]\% & -451.46 [-475.73,-427.99]\%\\
0.10 & 0.0000 & 1,517.19 & 0.6550 & 27,615.7 & 7,678.2 & 27,171.5 & -253.88 [-256.10,-251.74]\% & -255.22 [-268.98,-242.27]\%\\
0.15 & 0.0010 & 1,516.35 & 0.7575 & 20,038.5 & 7,716.0 & 19,838.2 & -157.10 [-159.17,-155.10]\% & -152.93 [-161.15,-145.21]\%\\
0.20 & 0.0090 & 1,504.90 & 0.8225 & 14,222.9 & 7,622.9 & 14,142.1 & -85.52 [-87.40,-83.74]\% & -80.90 [-86.00,-76.39]\%\\
0.30 & 0.9240 & 224.51 & 0.9205 & 4,814.2 & 4,690.3 & 1,245.5 & 73.44 [71.56,75.32]\% & 65.61 [63.25,67.84]\%\\
0.40 & 1.0000 & 1.00 & 0.9760 & 1,526.7 & 1,425.6 & 51.3 & 96.40 [96.28,96.52]\% & 65.07 [63.94,66.05]\%\\
0.50 & 1.0000 & 1.00 & 1.0000 & 54.0 & 50.0 & 50.0 & 0.00 [0.00,0.00]\% & 0.00 [0.00,0.00]\%\\
\bottomrule
\end{tabular}
\par\smallskip\noindent(a) Finite is the fraction of splits with a finite threshold $q$; infinite thresholds return the full available pool.
\end{table}

\begin{table}[htbp]\centering\small\setlength{\tabcolsep}{4pt}
\caption{Absolute-target cost ledger. All time entries are cumulative source-recorded seconds. New-fit savings add consensus fitting and one deployment to both policies. Query and time savings use their respective units.}
\label{tab:final-week-absolute-ledger}
\begin{tabular}{@{}rrrrrrrr@{}}
\toprule
$\varepsilon$ & Exact sec. & Serial sec. & Layer sec. & Deploy queries & Deploy sec. & New-fit query saving & New-fit time saving\\
\midrule
0.05 & 23,492,660 & 4,255,867 & 23,469,581 & 791.13 & 410,173 & -37.27\% & -44.71\%\\
0.10 & 14,577,321 & 4,099,711 & 14,563,061 & 528.66 & 248,828 & -21.62\% & -24.53\%\\
0.15 & 13,135,252 & 5,190,180 & 13,127,535 & 374.29 & 225,476 & -13.46\% & -18.15\%\\
0.20 & 10,434,048 & 5,745,924 & 10,394,344 & 269.89 & 184,409 & -7.26\% & -10.51\%\\
0.30 & 4,003,765 & 3,895,277 & 1,339,580 & 7.97 & 7,080 & 3.98\% & 6.05\%\\
0.40 & 204,596 & 192,364 & 67,185 & 1.00 & 1,248 & 1.65\% & 0.33\%\\
0.50 & 70,602 & 65,563 & 65,563 & 1.00 & 1,248 & 0.00\% & 0.00\%\\
\bottomrule
\end{tabular}
\par\smallskip Consensus fitting: 81,945.87 aggregate evaluations and 38,310,578 recorded seconds per split on average.
\end{table}

\subsection{Ranking, calibration, and deployment cost ledger}
\label{app:revision-cost}
Table~\ref{tab:revision-cost-ledger} separates the use of a supplied ranking from fitting a new consensus ranking, which adds the fitting cost to every policy. Deployment verification is charged at $N=1$ and $N=1000$ deployments, and the savings column gives percentages at these two counts. Candidate generation, PRM scoring, and the construction of target side information are shared by all policies and excluded. Query counts and recorded seconds appear in separate rows. The seconds sum three-fold model-fitting durations along the queries of each policy and measure total training work, while elapsed time depends on execution and parallelism.

\begin{table*}[t]
\centering\small
\caption{Measured cost ledger for the main settings. Query counts and cumulative source-recorded fitting seconds form separate ledger rows. Supplied-ranking savings include calibration and $D$ deployments; the last column additionally includes fitting a new TabRepo consensus ranking.}
\label{tab:revision-cost-ledger}
\setlength{\tabcolsep}{3pt}
\begin{tabular}{@{}lrrrrrr@{}}
\toprule
Task / ranking & Fitting & Serial cal. & Layer cal. & Deploy & \shortstack{Saving at\\$D=1$ / $D=1000$ (\%)} & New fit $D=1$\\
\midrule
\multicolumn{7}{@{}l}{\textit{verification queries}}\\
TREC / BM25 & --- & 260.69 & 173.51 & 4.23 & 32.91\% / 1.94\% & ---\\
MATH 1B / PRM & --- & 10,924.63 & 2,877.86 & 9.26 & 73.59\% / 39.86\% & ---\\
MATH 3B / PRM & --- & 8,385.05 & 1,522.91 & 4.96 & 81.79\% / 51.42\% & ---\\
TabRepo / random & --- & 118.68 & 100.87 & 1.06 & 14.87\% / 1.51\% & ---\\
TabRepo / consensus & 151,059.50 & 225.91 & 92.79 & 1.00 & 58.67\% / 10.86\% & 0.088\%\\
\addlinespace
\multicolumn{7}{@{}l}{\textit{source-recorded three-fold fit seconds}}\\
TabRepo / random & --- & 93,194.91 & 84,299.54 & 849.41 & 9.46\% / 0.94\% & ---\\
TabRepo / consensus & 121,896,820.72 & 380,482.84 & 182,179.35 & 1,877.74 & 51.86\% / 8.78\% & 0.162\%\\
\addlinespace
\bottomrule
\end{tabular}
\end{table*}

\subsection{Additional displays}
\label{app:additional-displays}
Figure~\ref{fig:acquisition} shows acquisition trajectories under the budget caps of Appendix~\ref{app:exp-policies}, and Figure~\ref{fig:cost} compares fixed-threshold costs and cost--size curves. Lines connect mean costs at the evaluated caps. On finite-$q$ splits, Layer reports its first finite threshold exactly when it recovers $q$, and its averaged curve slopes because splits finish at different budgets. Random-query can report a finite threshold before it recovers $q$. On the 2,000 displayed splits, its mean costs to a first finite threshold and to exact recovery are 418.69 and 756.77 on TREC, 3,037.95 and 12,969.97 on MATH 3B, and 109.92 and 794.51 on TabRepo consensus. Layer recovers $q$ with 172.75, 1,536.35, and 92.79 queries.
%% [FRONTIER] begin
Figure~\ref{fig:frontier} shows the simultaneous tolerance frontiers.
%% [FRONTIER] end

\begin{figure}[p]
\centering
\includegraphics[width=0.95\textwidth]{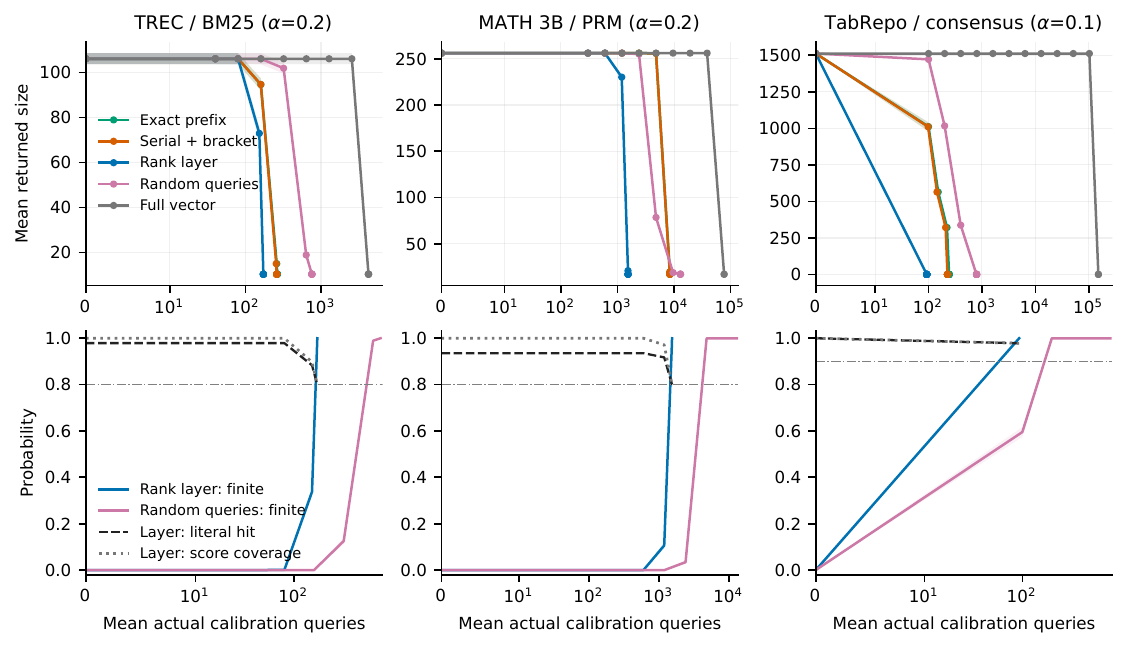}
\caption{Acquisition trajectories under budget caps, with 2,000 paired splits at the risk levels in the panel titles. Top: mean returned size, with pointwise 95\% split-bootstrap bands. Bottom: probability of a finite threshold, literal hit rate, and score coverage.}
\label{fig:acquisition}
\end{figure}
\begin{figure}[p]
\centering
\includegraphics[width=0.95\textwidth]{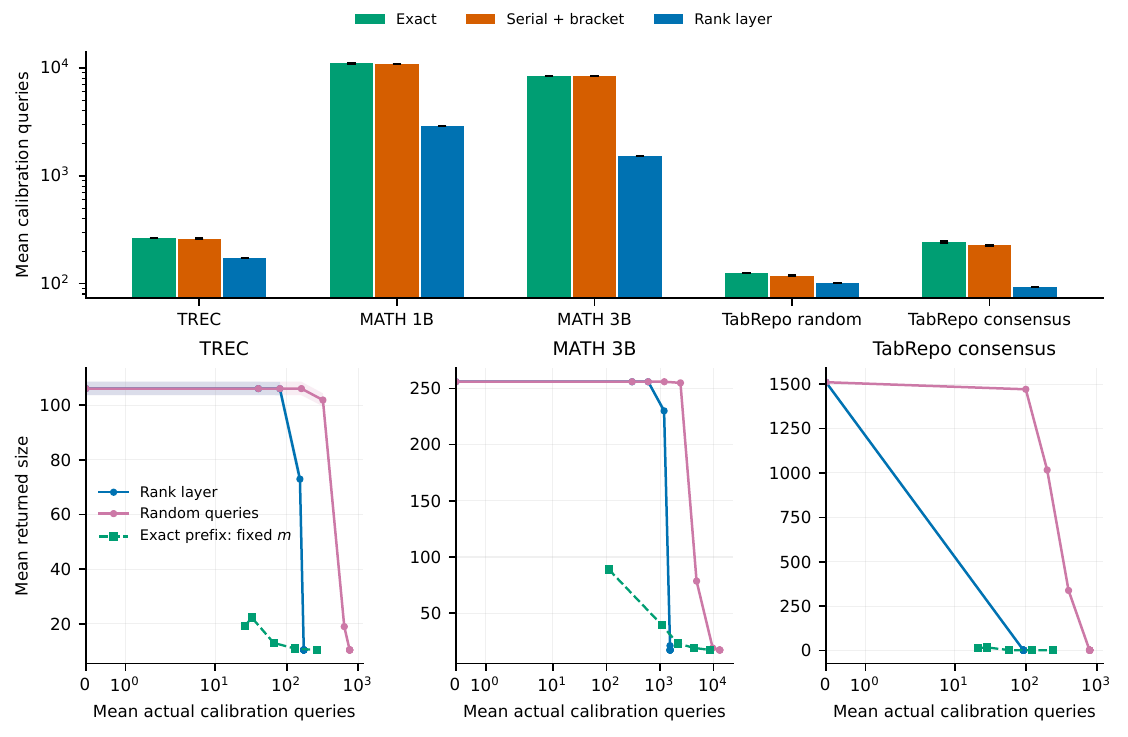}
\caption{Top: mean calibration cost of Exact, Serial, and Layer on the main settings, with 10,000 paired splits per setting. Bottom: cost--size curves on 2,000 splits, where Exact uses smaller calibration samples chosen before their verdicts are seen. Each sample size defines its own reference predictor. Error bars and bands are 95\% split-bootstrap intervals.}
\label{fig:cost}
\end{figure}
%% [FRONTIER] begin
\begin{figure}[p]
\centering
\includegraphics[width=0.95\textwidth]{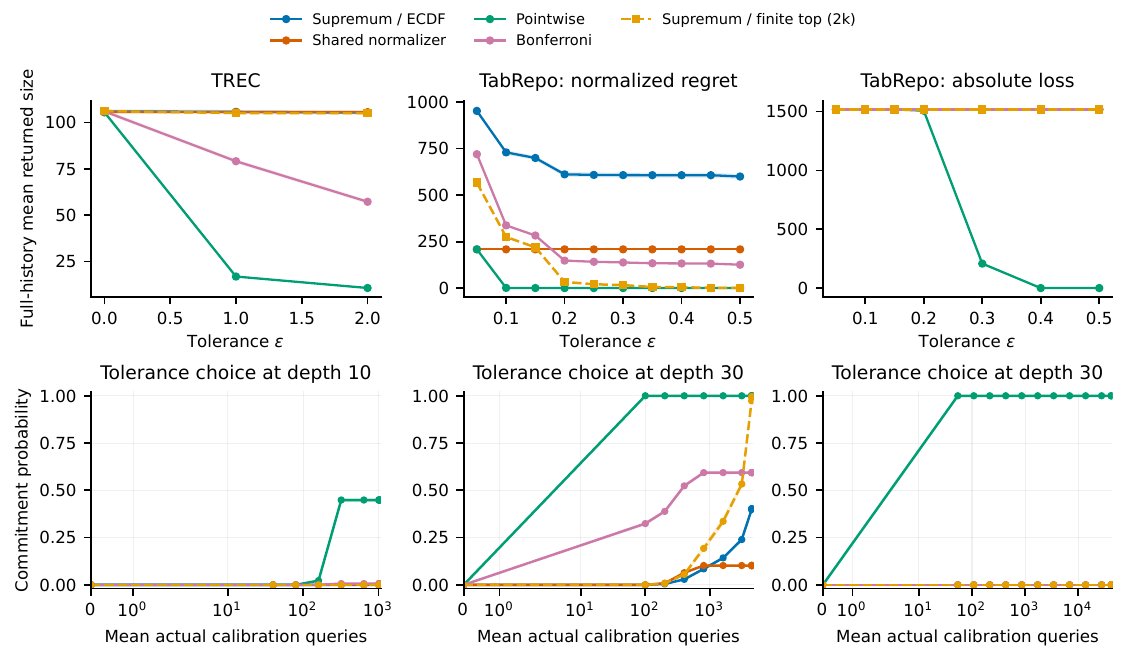}
\caption{Simultaneous tolerance frontiers at $\alpha=0.1$. ECDF supremum, shared normalizer, pointwise, and Bonferroni use 10,000 splits, and finite-top uses the first 2,000 of them. All constructions in a setting read the same query history. Bottom: probability of selecting a tolerance at deployment depth 10 for TREC and 30 for TabRepo. Shaded regions are pointwise 95\% intervals. Finite-top removes finite-score saturation.}
\label{fig:frontier}
\end{figure}
%% [FRONTIER] end
\end{document}